%% file: main.tex
\documentclass[journal]{IEEEtran}

\usepackage[cmex10]{amsmath}
\usepackage{amssymb,amsfonts}
\usepackage{amsthm}
\usepackage{graphicx}
\usepackage{booktabs,multirow}
\usepackage{algorithm}
\usepackage{algpseudocode}
\usepackage{xcolor}
\usepackage{tikz}
\usetikzlibrary{positioning,arrows.meta,fit,calc}
\usepackage{cite}
\usepackage{url}
\usepackage[caption=false,font=footnotesize]{subfig}
\usepackage{hyperref}
\hypersetup{hidelinks}

\graphicspath{{figs/}}

\input{macros}

\newtheorem{theorem}{Theorem}
\newtheorem{proposition}{Proposition}
\newtheorem{lemma}{Lemma}
\newtheorem{corollary}{Corollary}

\theoremstyle{definition}

\theoremstyle{remark}
\newtheorem{remark}{Remark}

\begin{document}

\title{Dynamics-Aware Adaptive Corridors with Feasibility-Perturbed Trust-Region SQP for Certified Nonholonomic Motion Planning}

\author{Yang Shi%
\thanks{Y. Shi is with Jiangsu College of Safety Technology, Jiangsu, China (e-mail: shiyang1@jscst.edu.cn).}}

\markboth{Preprint}%
{Shi: Dynamics-Aware Adaptive Corridors with Feasibility-Perturbed Trust-Region SQP}

\maketitle

\begin{abstract}
\input{sections/abstract}

\end{abstract}

\begin{IEEEkeywords}
Motion planning, trajectory optimization, sequential quadratic programming, trust region, safe corridors, autonomous parking.
\end{IEEEkeywords}

\input{sections/introduction}

\input{sections/related_work}

\input{sections/problem}

\input{sections/method}

\input{sections/theory}

\input{sections/experiments}

\input{sections/discussion}

\input{sections/conclusion}

\appendices
\input{sections/appendix_proofs}

\bibliographystyle{IEEEtran}
\bibliography{refs}

\end{document}

%% file: macros.tex
\newcommand{\ours}{DA-ASC/FP-SQP}
\newcommand{\obca}{OBCA}
\newcommand{\obcact}{OBCA-CT}
\newcommand{\stc}{STC-IPOPT}
\newcommand{\corrnlp}{Corridor-IPOPT}
\newcommand{\scvx}{SCvx}
\newcommand{\scvxrb}{SCvx-RB}

\newcommand{\state}{x}                 %
\newcommand{\ctrl}{u}                  %
\newcommand{\px}{p_x}
\newcommand{\py}{p_y}
\newcommand{\pos}{p}                   %
\newcommand{\head}{\theta}             %
\newcommand{\spd}{v}                   %
\newcommand{\steer}{\delta}            %
\newcommand{\acc}{a}                   %
\newcommand{\srate}{\omega}            %
\newcommand{\wb}{L}                    %
\newcommand{\Lf}{L_{\mathrm f}}        %
\newcommand{\Lr}{L_{\mathrm r}}        %
\newcommand{\wid}{W}                   %
\newcommand{\ehead}[1]{e(#1)}          %
\newcommand{\eperp}[1]{e^{\perp}(#1)}  %
\newcommand{\Rot}[1]{R(#1)}

\newcommand{\Nint}{N}                  %
\newcommand{\Tref}{T_0}                %
\newcommand{\hzero}{h_0}               %
\newcommand{\tscale}{\sigma}           %
\newcommand{\Thor}{T}                  %
\newcommand{\hstep}{h}                 %
\newcommand{\Fmap}{F}                  %
\newcommand{\tscmin}{\sigma_{\min}}
\newcommand{\tscmax}{\sigma_{\max}}

\newcommand{\bodypt}[1]{\beta_{#1}}    %
\newcommand{\corner}[1]{c_{#1}}        %
\newcommand{\Obs}[1]{\mathcal{O}_{#1}} %
\newcommand{\Map}{\mathcal{B}}         %
\newcommand{\Corr}[1]{\mathcal{P}_{#1}}%
\newcommand{\CA}[1]{A_{#1}}            %
\newcommand{\Cb}[1]{b_{#1}}            %
\newcommand{\Mb}[1]{M_{#1}}            %
\newcommand{\shpoly}{SH polygon}        %
\newcommand{\shR}{R_{\mathrm{SH}}}      %
\newcommand{\sweep}[1]{\varepsilon_{#1}}                 %
\newcommand{\sweeps}[1]{\varepsilon^{\mathrm{s}}_{#1}}   %
\newcommand{\clr}{\mu}                 %
\newcommand{\gcon}{g}                  %
\newcommand{\Xf}{\mathcal{X}_{\mathrm{f}}} %

\newcommand{\cost}{J}
\newcommand{\dec}{z}                   %
\newcommand{\Xtraj}{X}
\newcommand{\Utraj}{U}
\newcommand{\step}{\Delta}             %
\newcommand{\trr}{r}                   %
\newcommand{\tlev}{t}                  %
\newcommand{\Dscale}{D}                %
\newcommand{\kQ}{\kappa_Q}             %
\newcommand{\restr}{\psi}              %
\newcommand{\Klqr}[1]{K_{#1}}          %
\newcommand{\ratio}{\rho}              %
\newcommand{\tauv}{\tau_{\mathrm v}}   %
\newcommand{\taus}{\tau_{\mathrm s}}   %
\newcommand{\infeas}{\mathcal{V}}      %
\newcommand{\crit}{\chi}               %
\newcommand{\trans}{\Phi}              %
\newcommand{\Hess}{H}
\newcommand{\budget}{B}                %
\newcommand{\sgain}{s_{\mathrm{gain}}}  %
\newcommand{\sact}{s_{\mathrm{act}}}    %

\newcommand{\flow}[1]{\varphi_{#1}}      %
\newcommand{\rkerr}[1]{\eta_{#1}}        %
\newcommand{\betamax}{\bar\beta}         %
\newcommand{\kproj}{\kappa_{\mathrm p}}  %
\newcommand{\Ltwo}{L_2}                  %
\newcommand{\tsum}{\Lambda}              %
\newcommand{\kQs}{\kappa_Q^{\star}}      %
\newcommand{\Body}{\mathcal{R}}          %
\newcommand{\ghat}{g_0}                  %

\newcommand{\lamo}[1]{\lambda_{#1}}    %
\newcommand{\zetao}{\zeta}             %
\newcommand{\dmin}{d_{\min}}
\newcommand{\wvirt}{w_{\nu}}           %

\newcommand{\ninf}[1]{\left\lVert #1 \right\rVert_{\infty}}

%% file: sections/abstract.tex
Optimisation-based parking planners usually impose collision constraints only at the time samples, so a vehicle corner can cut an obstacle between samples, and no executable trajectory exists until the solver converges. We present a planner for car-like vehicles with reverse gear in which every iterate of the optimisation phase satisfies the discretised dynamics exactly and keeps the whole vehicle rectangle clear of obstacles between the samples. Each time interval receives one convex corridor that holds all vehicle corners at both ends and is shrunk by a sweep margin bounding how far the corner paths leave their chords. Separating half-planes give the corridors a direction out of obstacles when the initial guess is in collision; later, heading-aligned boxes are grown from the speed, curvature and step of the current iterate and rebuilt after accepted steps. A feasibility-perturbed trust-region sequential quadratic programming method projects each step onto the dynamics by feedback and verifies it exactly; the cost decreases monotonically, and once the corridors stop changing, limit points are Karush--Kuhn--Tucker points of the corridor-constrained problem or violate a constraint qualification. On 820 benchmark cases the planner succeeds in 818 without penetration (797 from the first initial guess), none of the 7103 evaluated optimisation-phase iterates is unusable, and it succeeds in 96.5\% of the cases when 99\% of the initial guesses intersect an obstacle. Its manoeuvres take 0.7\% longer in the median than those of a similarly certified exact-collision baseline. The guarantees hold for the planning model, not for a physical vehicle.

%% file: sections/introduction.tex
\section{Introduction}\label{sec:intro}
\IEEEPARstart{P}{arking} lots, loading bays and narrow passages force a car-like vehicle to plan close to obstacles, often with several changes between forward and reverse gear. Most planners for these scenes work in two stages~\cite{Guo2024Survey,Li2023TPCAP}: a search-based front end such as Hybrid~A*~\cite{Dolgov2010HybridAstar} finds a coarse path, and a numerical optimal-control back end turns it into a trajectory that respects the vehicle dynamics. The back end discretises time and imposes the collision constraints at the samples. Between two samples a corner of the vehicle moves along an arc that the constraints do not see, and a minimum-time objective drives the trajectory against the obstacles, where such an arc can cut an obstacle vertex. In our benchmark of 820 cases, the exact dual formulation \obca{} with a minimum distance of 0.10~m returns 76 trajectories that penetrate an obstacle between samples (Section~\ref{sec:e-main}). %

The two main families of back ends pay for their collision model in different ways. Exact formulations~\cite{Zhang2021OBCA,Zhang2018HOBCA,Li2015Unified} give nonconvex nonlinear programs whose size grows with the number of obstacles; they hold at the samples only, and an interrupted solve does not in general return an executable trajectory. Corridor formulations~\cite{Li2020STC,Li2022LIOM,Guo2025TPCKC} replace the obstacles by convex sets around the front-end path and obtain linear constraints, but the sets are built from geometry around a reference that the front end supplies, so the back end inherits the quality of that reference. When we perturb the front-end path so that 99\% of the warm starts intersect an obstacle, our re-implementation of a safe-travel-corridor planner, solved by IPOPT, succeeds in 14.1\% of the cases (Section~\ref{sec:e-init}). %

Sequential convex programming (SCP) and trust-region methods, which solve the back end as a sequence of convex subproblems, split according to what their iterates satisfy. Penalty methods such as SCvx~\cite{Mao2016SCvx,Mao2019SCvx,Luo2024Remarks}, applied to parking by Boyali and Thompson~\cite{Boyali2020Parking}, carry violations of the dynamics and of the constraints until the penalty vanishes. Feasible-set methods such as CFS~\cite{Liu2018CFS} keep the iterates feasible at the samples but exclude nonlinear dynamics, and their vehicle version SLiFS~\cite{Sun2022SLiFS} penalises the dynamics again. STO~\cite{Yu2025STO} rebuilds its corridors at every SQP iteration without keeping the iterates feasible, and the anytime optimiser of Chen et al.~\cite{Chen2025Anytime} repeats complete nonlinear solves until a time limit. Among the planners in Table~\ref{tab:related}, none certifies inter-sample collision avoidance of every iterate while its convex sets change. That property is what a planner needs if it is to be stopped early: the last iterate is then a trajectory the vehicle can execute, not an intermediate point of an optimisation.

We provide this property with two components. DA-ASC builds one convex corridor per interval from the current iterate, holding all four corners of the vehicle rectangle at both ends of the interval, and shrinks it by a sweep margin that bounds how far the RK4 motion can leave the chord. FP-SQP solves a restricted quadratic subproblem inside these corridors, projects the step onto the dynamics by time-varying LQR feedback, verifies the result exactly and accepts it by a trust-region ratio test. Our earlier work used line-search and trust-region SCP with pseudospectral discretisation for longitudinal on-ramp merging~\cite{Shi2022PCO,Shi2023PCOMPC} and for entry trajectories~\cite{Wang2020ImprovedSCP}, with slack variables and $\ell_1$ merit functions and without obstacle geometry; the present planner keeps its iterates feasible instead. The contributions are as follows.
\begin{enumerate}
\item A shared convex corridor per interval with a sweep margin that certifies collision avoidance of the whole vehicle rectangle between the samples, along the motion of each interval, not only of its corners at the samples (Proposition~\ref{prop:sweep}, Corollary~\ref{cor:shared}).
\item Corridors of two kinds: separating-hyperplane (SH) polygons, which give the feasibility-restoration phase (Phase~I) a direction out of the obstacles when the warm start is in collision, and heading-aligned DA-ASC boxes for the optimisation phase (Phase~II), grown with targets from the speed, curvature and step of the current iterate and rebuilt only after accepted steps, finitely often (Section~\ref{sec:m-dasc}, Lemmas~\ref{lem:manage} and~\ref{lem:finite}).
\item FP-SQP, an adaptation of the feasibility-perturbed trust-region SQP of Wright and Tenny~\cite{Wright2004FPSQP} to corridors that change: every Phase-II iterate satisfies the RK4 dynamics and its interval motions are collision-free, the cost decreases monotonically without a penalty weight, and once the corridors stop changing every limit point is a KKT point of the corridor-constrained problem or violates the Mangasarian--Fromovitz constraint qualification (MFCQ) (Theorem~\ref{thm:fpsqp}).
\item Certified anytime behaviour and linear per-iteration cost: none of the 7103 Phase-II iterates evaluated on 139 cases is unusable, and the subproblem is an optimal-control-structured quadratic program whose cost per iteration grows linearly with the horizon (Sections~\ref{sec:e-anytime} and~\ref{sec:e-cost}). %
\item \obcact{}, a variant of \obca{} that shares one separating hyperplane per interval and adds the sweep margin, which makes the exact dual formulation certifiable between samples and serves as the strongest baseline (Proposition~\ref{prop:obcact}).
\item An evaluation on 820 cases against six baselines, all re-implemented with the same front end, vehicle model and independent collision check, including warm starts in collision and a closed-loop simulation that shows how the clearance buffer and an input margin are chosen for execution (Section~\ref{sec:exp}).
\end{enumerate}
Our planner solves 818 of the 820 cases without penetration, against 791 for \obcact{} with a 60~s limit and 809 with a 3600~s limit; with the first warm start alone it solves 797 cases, and its manoeuvres take 0.7\% longer than those of \obcact{} in the median over 791 pairs. %
The guarantees hold for the planning model; the Python prototype takes a median of 5.8~s per case, and we make no claim about real-time execution or about a physical vehicle. %

Section~\ref{sec:related} reviews related work, Section~\ref{sec:problem} states the problem, Section~\ref{sec:method} describes the planner, Section~\ref{sec:theory} gives its properties, Section~\ref{sec:exp} reports the experiments, and Section~\ref{sec:discussion} discusses the limitations.

%% file: sections/related_work.tex
\section{Related Work}\label{sec:related}

\begin{table*}[!t]
\centering
\caption{Optimisation back ends for low-speed planning with obstacles, compared by the properties this paper targets. ``Samples'': collision constraints hold only at the samples. ``Usable iterates'': every iterate of the optimisation can be executed without collision and satisfies the dynamics. A dash means that the property is not claimed; ``local solver'': convergence is that of the nonlinear programming solver used.}
\label{tab:related}
\footnotesize
\setlength{\tabcolsep}{4pt}
\begin{tabular}{@{}p{3.0cm}p{2.6cm}p{4.1cm}p{2.3cm}p{2.1cm}p{2.6cm}@{}}
\toprule
Planner & Body model & Convex set and its update & Inter-sample & Usable iterates & Convergence result \\
\midrule
OBCA, H-OBCA~\cite{Zhang2021OBCA,Zhang2018HOBCA} & rectangle & none (exact dual constraints) & samples & -- & local solver \\
STC~\cite{Li2020STC} & two discs & one box per disc and sample, built once around the path & samples & -- & local solver \\
LIOM~\cite{Li2022LIOM} & two discs & boxes rebuilt from each intermediate solution & samples & -- (dynamics penalised) & -- \\
Embodied Footprints~\cite{Li2023Embodied} & rectangle & none; footprint buffers as decision variables & buffers from speed and curvature & -- & local solver \\
Chen et al.~\cite{Chen2025Anytime} & rectangle & obstacle subset selected by a trust region & samples & -- & -- \\
STO~\cite{Yu2025STO} & rectangle & polygon per sample, rebuilt at every SQP iteration & samples & -- & -- \\
CFS~\cite{Liu2018CFS} & general & convex feasible set around each iterate & samples & yes (no equality constraints) & local optimum \\
SLiFS~\cite{Sun2022SLiFS} & discs & convex feasible sets, outer loop & samples & -- (dynamics penalised) & -- \\
SCvx~\cite{Mao2016SCvx,Mao2019SCvx,Luo2024Remarks} & general & convexified constraints & samples & -- (virtual control) & stationary point \\
\midrule
OBCA-CT (Prop.~\ref{prop:obcact}) & rectangle & one separating hyperplane per interval & sweep margin & -- & local solver \\
\ours{} (ours) & rectangle & one corridor per interval, rebuilt after accepted steps, finitely often & sweep margin & yes (Phase~II, interval motions) & KKT or MFCQ fails (Thm.~\ref{thm:fpsqp}) \\
\bottomrule
\end{tabular}
\end{table*}

Table~\ref{tab:related} places our planner among the optimisation back ends that a reviewer of low-speed planning is likely to compare it with. The three subsections below follow its columns.

\subsection{Collision Formulations and Inter-Sample Safety}\label{sec:rw-collision}
Exact collision formulations constrain the full vehicle shape against every nearby obstacle. \obca{}~\cite{Zhang2021OBCA} writes the distance between two polytopes through dual variables and makes the constraint smooth; H-OBCA~\cite{Zhang2018HOBCA} warm-starts it from Hybrid~A*~\cite{Dolgov2010HybridAstar}, and TDR-OBCA~\cite{He2021TDROBCA} adds temporal and dual-variable warm starts and reformulates the problem. Li and Shao~\cite{Li2015Unified} use a polygon-area test instead. All of these impose the constraint at the samples of the discretised problem, and the number of constraints grows with the number of obstacles. With the step lengths that minimum-time parking produces, a corner can cross an obstacle vertex between two samples: in our benchmark, \obca{} with $\dmin=0.10$~m returns 76 penetrating trajectories in 820 cases (Section~\ref{sec:e-main}). %

Three ideas address the gap between samples. Embodied Footprints~\cite{Li2023Embodied} enlarge the footprint by buffers that depend on speed and curvature and optimise the buffers with the trajectory; the analysis assumes forward driving and states that reverse driving is symmetric. Allamaa et al.~\cite{Allamaa2023Envelope} bound the collocation polynomial of an orthogonal-collocation transcription by a convex envelope, so that point constraints hold along the whole interval. Swept-volume methods~\cite{Li2025Swept} constrain the region swept by the body directly. Our sweep margin belongs to the first group, with a different mechanism: the four corners at both ends of an interval lie in one convex corridor, so the chord of every corner path lies in it as well, and only the arc-to-chord deviation of the RK4 motion remains to be bounded (Proposition~\ref{prop:sweep}). The margin is a smooth function of the state and the time scale, so the subproblem stays a quadratic program; the same argument applied to one separating hyperplane per interval gives the certified variant \obcact{} (Proposition~\ref{prop:obcact}), which we use as the strongest baseline.

\subsection{Corridor Construction}\label{sec:rw-corridor}
Corridor methods replace the obstacles by a convex set around a reference, which turns the collision constraint into linear inequalities. The safe travel corridors of Li et al.~\cite{Li2020STC} cover the vehicle by two discs, dilate the obstacles by the disc radius, and give each disc centre an axis-aligned box per sample, grown around the front-end path once. LIOM~\cite{Li2022LIOM} keeps the disc-and-box model but rebuilds the boxes from each intermediate solution, softening the dynamics by a penalty so that each intermediate problem has only box constraints; in its hardware test, 0.01~m is added to the disc radius to absorb small tracking errors. Liu et al.~\cite{Liu2026Corridor} combine corridor constraints with parking-slot matching, and TPCKC~\cite{Guo2025TPCKC} adds vertex-to-polytope constraints only for the obstacles that the current solution hits. Schäfer et al.~\cite{Schafer2023Solution} derive driving corridors from reachable sets on structured roads. STO~\cite{Yu2025STO}, the closest method in structure, grows a polygon around each reference point by GJK-accelerated ellipse shrinking and expansion and rebuilds it at every SQP iteration, including across gear shifts. None of these corridor planners claims that its iterates are usable before convergence.

In robotics, safe flight corridors~\cite{Liu2017SFC}, IRIS~\cite{Deits2015IRIS} and FIRI~\cite{Wang2024FIRI} build large convex polytopes around ellipsoids; FIRI calls a region manageable when it contains a prescribed seed set, and our Lemma~\ref{lem:manage} uses the term in that sense. Linear safe corridors~\cite{Park2022LSC} give a quadrotor swarm continuous-time safety through the convex-hull property of Bernstein polynomials, and SFC-constrained SCP has been applied to fixed-wing aircraft~\cite{Sun2025SFCSCP}. The \shpoly{} of Section~\ref{sec:method} is simpler than these regions: it intersects a heading-aligned box with one supporting half-plane per nearby obstacle and does no ellipsoid inflation. We use it because, when the vehicle overlaps an obstacle, the normal of that half-plane points out of the obstacle, which gives Phase~I a descent direction from warm starts that are in collision.

\subsection{Sequential Convex Programming and Trust-Region Convergence}\label{sec:rw-scp}
Sequential convex programming methods differ in whether their iterates are feasible. SCvx~\cite{Mao2016SCvx,Mao2019SCvx} relaxes the linearised dynamics by a virtual control with an exact penalty and proves convergence to a stationary point; Luo et al.~\cite{Luo2024Remarks} identified errors in the original convergence proof and re-established convergence under stricter assumptions, and we refer to SCvx convergence in that form. GuSTO~\cite{Bonalli2019GuSTO,Bonalli2023SCP} updates the trust region and the penalty weight jointly and shows convergence in continuous time, and continuous-time SCvx~\cite{Elango2025ctSCvx} enforces path constraints between nodes through an integrated exterior penalty. Boyali and Thompson~\cite{Boyali2020Parking} applied successive convexification to parking. In all penalty-based methods an iterate may violate the dynamics or the constraints until the penalty vanishes, so interrupting them yields no usable trajectory; our SCvx baseline also needs its penalty weight tuned (Section~\ref{sec:e-setup}).

Feasible methods keep the constraints satisfied along the way. CFS~\cite{Liu2018CFS} solves a sequence of problems on a convex feasible set around the current iterate and converges to a (weak) local optimum, but its assumptions exclude equality constraints and therefore nonlinear dynamics; SLiFS~\cite{Sun2022SLiFS} brings the idea to vehicles with discs and handles the nonlinear dynamics by an $\ell_1$ penalty and a trust region. Sequential convex restriction~\cite{Scheffe2023SCR} keeps racing trajectories feasible by restricting the track to convex polygons. Feasible SQP in nonlinear programming~\cite{Panier1993FSQP,Lawrence2001FSQP} tilts the search direction into the feasible set; the feasibility-perturbed trust-region SQP of Wright and Tenny~\cite{Wright2004FPSQP,Tenny2004NMPC} instead perturbs each step back to feasibility and proves global convergence under a restriction of the subproblem. We adapt the latter, with Hauser's projection operator~\cite{Hauser2002Projection} as the perturbation, and extend its analysis to corridors that change after accepted steps.

``Trust region'' also has other meanings in vehicle planning. Chen et al.~\cite{Chen2025Anytime} use a trust region to select the obstacles that enter the constraints, and their anytime loop repeats complete nonlinear solves until a time limit; Bu et al.~\cite{Bu2026AdaptiveTR} adapt a trust region around local terrain approximations; for spacecraft, Bernardini et al.~\cite{Bernardini2024StateTR} make the radius state-dependent. Jaitly et al.~\cite{Jaitly2025Ellipsoidal} express ellipsoidal corridors as projections of orthogonal trust regions. Among the planners in Table~\ref{tab:related}, none certifies inter-sample collision avoidance of every iterate while its convex sets change.

%% file: sections/problem.tex
\section{Problem Formulation}\label{sec:problem}

\subsection{Vehicle Model and Time Grid}
The vehicle is a kinematic bicycle with the reference point at the centre of the rear axle, which is adequate for the low speeds of parking and manoeuvring~\cite{Polack2017Bicycle,Kong2015Models}. The state $\state=(\px,\py,\head,\spd,\steer)$ collects position, heading, signed speed and front-wheel steering angle; the input $\ctrl=(\acc,\srate)$ collects acceleration and steering rate:
\begin{equation}\label{eq:ct}
\begin{aligned}
&\dot\px=\spd\cos\head,\quad \dot\py=\spd\sin\head,\quad \dot\head=\spd\tan\steer/\wb,\\
&\dot\spd=\acc,\quad \dot\steer=\srate .
\end{aligned}
\end{equation}
Reverse gear needs no integer variable: $\spd$ changes sign continuously, and a gear change is a zero crossing of $\spd$ whose timing the optimiser can move. The horizon is split into $\Nint$ intervals of equal length $\hstep=\tscale\hzero$, where $\hzero=\Tref/\Nint$ comes from the warm start and the time-scaling variable $\tscale\in[\tscmin,\tscmax]$ sets the horizon $\Thor=\tscale\Tref$. The input is held constant on each interval, and one RK4 step gives the transition map
\begin{equation}\label{eq:model}
\state_{k+1}=\Fmap(\state_k,\ctrl_k,\tscale\hzero),\qquad k=0,\dots,\Nint-1 .
\end{equation}
Speed, steering angle, acceleration and steering rate are bounded by $\spd_{\max}$, $\steer_{\max}<\pi/2$, $\acc_{\max}$ and $\srate_{\max}$ (Table~\ref{tab:params}).

\subsection{Vehicle Body, Obstacles and Map}
The vehicle occupies a rectangle of wheelbase $\wb$, front overhang $\Lf$, rear overhang $\Lr$ and width $\wid$. Its four corners have body-frame offsets $\bodypt{j}$, $j=1,\dots,4$, from the rear axle and world positions $\corner{j}(\state)=\pos+\Rot{\head}\bodypt{j}$ with $\pos=(\px,\py)$; the occupied set at state $\state$ is the convex hull of these corners. Obstacles are closed convex polygons $\Obs{m}$, $m=1,\dots,M$; non-convex obstacles are split into triangles by ear clipping. The vehicle must stay inside the map box $\Map$. We model the whole rectangle rather than a disc cover: a cover by three discs extends up to 0.275~m beyond the sides and 0.465~m beyond the bumpers, whereas the rectangle at the start or goal pose of the TPCAP benchmark~\cite{Li2023TPCAP} clears the nearest obstacle by less than 0.465~m in 16 of the 20 cases, by less than 0.275~m in 6, and by 0.148~m in the tightest case. %

\subsection{Planning Problem}
The vehicle starts at rest at $\state_0=\state_{\mathrm{init}}$ and must end in the terminal box $\Xf=\{\state:|\state-\state_{\mathrm{goal}}|\le\bar\state\}$ (componentwise), which bounds position, heading and speed and leaves the final steering angle free. A box rather than an equality is used because the projection of Section~\ref{sec:m-accept} satisfies the dynamics exactly but moves the final state by a second-order amount. The objective trades time against control effort,
\begin{equation}\label{eq:cost}
\cost(\Utraj,\tscale)=w_T\tscale\Tref+\tscale\hzero\sum_{k=0}^{\Nint-1}\big(w_a\acc_k^2+w_\omega\srate_k^2\big),
\end{equation}
and the planning problem is
\begin{equation}\label{eq:P}
\begin{aligned}
\min_{\Xtraj,\Utraj,\tscale}\ &\cost(\Utraj,\tscale)\quad\text{s.t. \eqref{eq:model}, the bounds, } \state_0=\state_{\mathrm{init}},\ \state_\Nint\in\Xf,\\
&\operatorname{conv}\{\corner{j}(\state(t))\}_{j=1}^4\subset\Map\setminus\textstyle\bigcup_m\Obs{m}\quad\forall t\in[0,\Thor],
\end{aligned}
\end{equation}
where, for $t\in[k\hstep,(k+1)\hstep]$, $\state(t)$ is the \emph{interval motion}: the solution of \eqref{eq:ct} under $\ctrl_k$ started at the sampled state $\state_k$. The last constraint is imposed for every $t$, not only at the samples; Section~\ref{sec:method} replaces it by finitely many linear constraints that imply it. One RK4 step is not exact, so an interval motion ends close to $\state_{k+1}$ but not at it, and the interval motions do not join into one solution of \eqref{eq:ct}. That solution, started at $\state_0$ and driven by the whole input sequence, is the \emph{executed motion}; Remark~\ref{rem:rk4} bounds its distance from the interval motions.

A trajectory is \emph{usable} if its interval motions keep the rectangle clear of all obstacles and inside the map for all $t$, it respects the bounds and it ends in $\Xf$. A benchmark run counts as a \emph{success} if the planner terminates and an independent check (Section~\ref{sec:e-setup}) confirms that the returned trajectory is usable.

%% file: sections/method.tex
\section{DA-ASC Corridors and FP-SQP}\label{sec:method}

Figure~\ref{fig:overview} and Algorithm~\ref{alg:main} summarise the planner. A Hybrid A* search with Reeds--Shepp expansions~\cite{Dolgov2010HybridAstar,Reeds1990RS}, run on the same rectangle, returns a collision-free path, and a trapezoidal speed profile on each gear segment turns it into the warm start $(\Xtraj^0,\Utraj^0,\Tref)$; we call the two together the front end. Phase~I turns the warm start, which violates the dynamics and often the corridors, into a feasible trajectory. Phase~II then optimises \eqref{eq:cost} while every accepted iterate stays feasible.

\begin{figure}[t]
\centering
\begin{tikzpicture}[font=\scriptsize, node distance=4mm and 4.5mm,
  blk/.style={draw, rounded corners=1.5pt, align=center, inner sep=2.5pt, minimum height=7mm, text width=19mm, fill=white},
  ours/.style={blk, draw=blue!60!black},
  arr/.style={-{Stealth[length=1.6mm]}, semithick}]
\node[blk] (fe) {front end:\\Hybrid A*, profile};
\node[ours, right=of fe] (p1) {Phase I:\\min $\infeas$ on SH polygons};
\node[ours, right=of p1] (pv) {projection,\\verification};
\node[ours, below=7mm of pv] (sub) {subproblem \eqref{eq:sub}\\(HPIPM)};
\node[ours, left=of sub] (proj) {projection \eqref{eq:proj},\\verification \eqref{eq:verify}};
\node[ours, left=of proj] (acc) {ratio test,\\rebuild \eqref{eq:rebuild}};
\draw[arr] (fe) -- (p1);
\draw[arr] (p1) -- (pv);
\draw[arr] (pv) -- (sub);
\draw[arr] (sub) -- (proj);
\draw[arr] (proj) -- (acc);
\draw[arr] (acc.south) -- ++(0,-3mm) -| (sub.south);
\node[below=5.5mm of proj, font=\tiny] {every accepted iterate is certified};
\node[draw=blue!60!black, dashed, rounded corners, inner xsep=1.5mm, inner ysep=2.2mm, fit=(sub)(proj)(acc)] (ph2) {};
\node[anchor=south west, font=\scriptsize, inner sep=1pt] at (ph2.north west) {Phase II};
\end{tikzpicture}
\caption{Structure of DA-ASC/FP-SQP. Phase~I restores feasibility on SH polygons; in Phase~II every accepted iterate is dynamically exact and collision-free between samples, so the loop can be interrupted after any accepted step.}
\label{fig:overview}
\end{figure}
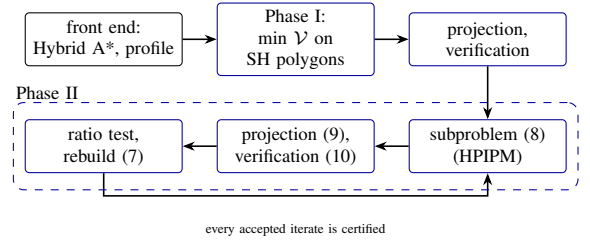

\subsection{Shared Convex Corridors and the Sweep Margin}\label{sec:m-corridor}
Each interval $k$ has one corridor $\Corr{k}=\{\pos:\CA{k}\pos\le\Cb{k}\}$, a convex polygon with unit-norm rows $a_i^\top$, that must hold all four corners at both ends of the interval (Fig.~\ref{fig:corridor}a). With the sweep margin of Proposition~\ref{prop:sweep} and the clearance buffer $\clr$, the corridor constraints are
\begin{equation}\label{eq:corr}
\gcon_{k,j,q,i}(\dec)=a_i^\top\corner{j}(\state_q)+\sweeps{k,j}+\clr-b_i\le0,\qquad q\in\{k,k+1\},
\end{equation}
for every corner $j$ and row $i$. The margin $\sweep{k,j}=\Mb{k}(\bodypt{j})\hstep^2/8$ of \eqref{eq:sweep} contains maxima and absolute values, so the algorithm uses a smooth upper bound $\sweeps{k,j}$. With $\epsilon_0=10^{-3}$, let
\begin{equation}\label{eq:epss}
\begin{aligned}
\bar V_k^2&=\spd_k^2+\spd_{k+1}^2,\qquad \bar K_k^2=(\tan^2\steer_k+\tan^2\steer_{k+1})/\wb^2,\\
\bar S_k&=\sec^2\steer_k+\sec^2\steer_{k+1}-1,\qquad \langle x\rangle=(x^2+\epsilon_0^2)^{1/2}.
\end{aligned}
\end{equation}
$\sweeps{k,j}$ is \eqref{eq:sweep} with $\Theta_{1,k}$ replaced by $\bar V_k^2\bar K_k^2$, $V_k^2K_k$ by $\bar V_k^2\langle\bar K_k\rangle$, $\Theta_{2,k}$ by $\langle\acc_k\rangle\langle\bar K_k\rangle+\langle\bar V_k\rangle\langle\srate_k\rangle\bar S_k/\wb$ and $|\acc_k|$ by $\langle\acc_k\rangle$. Each replacement is at least the term it replaces, and $\Mb{k}$ is nondecreasing in these terms, so $\sweeps{k,j}\ge\sweep{k,j}$. Every square root has a positive argument, and the first component of $\Mb{k}$ is at least $\epsilon_0$, so $\sweeps{k,j}$ is infinitely differentiable for $|\steer_k|,|\steer_{k+1}|<\pi/2$. %
 The same $\sweeps{k,j}$ is used in Phase~I, in the subproblem and in the verification, so that a feasible iterate makes $\step=0$ feasible for the subproblem. With $\hstep=\tscale\hzero$, $\partial\sweeps{k,j}/\partial\tscale=2\sweeps{k,j}/\tscale$ exactly.

For fixed corridors, the problem solved in Phase~II is
\begin{equation}\label{eq:PF}
\begin{aligned}
\min_{\dec=(\Xtraj,\Utraj,\tscale)}\ &\cost(\Utraj,\tscale)\\
\text{s.t. }&\text{\eqref{eq:model}, }\state_0=\state_{\mathrm{init}},\ \text{bounds},\ \state_\Nint\in\Xf,\ \eqref{eq:corr}.
\end{aligned}\tag{P$_\mathrm F$}
\end{equation}
By Corollary~\ref{cor:shared} and Remark~\ref{rem:rk4}, every feasible point of \eqref{eq:PF} with valid corridors whose RK4 local errors satisfy $\rkerr{k}<\clr-\tauv$ is a feasible point of \eqref{eq:P}: the non-convex constraint for all $t$ has become a finite set of constraints that are linear in the corner positions.

\begin{figure}[t]
\centering
\includegraphics[width=\columnwidth]{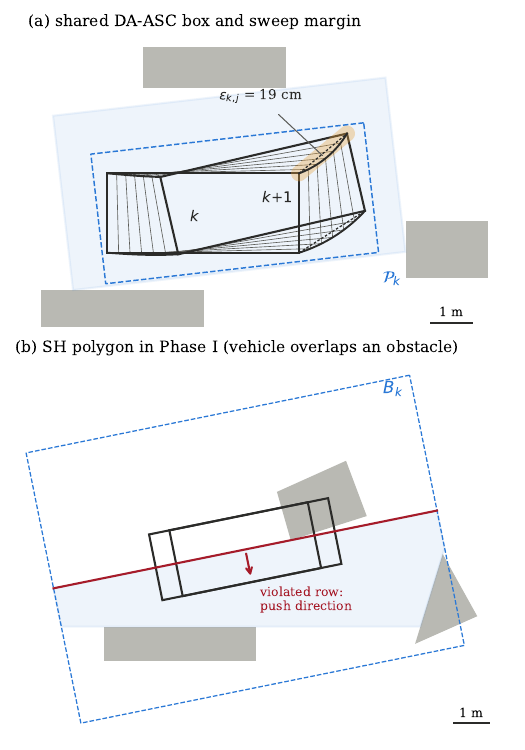}
\caption{Corridors of one interval. (a) A DA-ASC box $\Corr{k}$ (shaded) grown from the seed box (dashed) holds all four corners at $k$ and $k+1$; each corner moves along an arc (solid) that stays within $\sweep{k,j}$ of its chord (dotted), shown for the front-left corner. The step is 1.0~s here to make the arcs visible; the planner's steps do not exceed 0.31~s. (b) An SH polygon (shaded) during Phase~I: the bounding box $B_k$ (dashed) is cut by one half-plane per obstacle. The vehicle overlaps the upper obstacle, so the corresponding row is violated and its normal gives the direction in which Phase~I moves the vehicle.}
\label{fig:corridor}
\end{figure}

\subsection{Corridor Generation: SH Polygons and DA-ASC Boxes}\label{sec:m-dasc}
Both corridor types start from the eight corners of the interval, each inflated by its own margin $\sweeps{k,j}+\clr$, and from the orientation $\phi_k$, the circular mean of $\head_k$ and $\head_{k+1}$. The body axis is used instead of the direction of motion because it stays defined at a cusp, where the vehicle stops, and a rectangle is invariant under $\phi\mapsto\phi+\pi$, so forward and reverse segments are treated alike.

\emph{SH polygons.} An SH polygon (Fig.~\ref{fig:corridor}b) is the intersection of a heading-aligned bounding box $B_k$, the smallest box around the inflated corners grown by $\shR$ on every side, with one half-plane per obstacle that meets $B_k$; the map boundary counts as four obstacles. For each obstacle the half-plane normal is chosen among the obstacle's edge normals and the edge normals of the convex hull of the corners as the one with the largest signed gap between corners and obstacle, and its boundary line is placed on the obstacle's support line. If the inflated corners are clear of the obstacle, this normal separates them from it and the corners satisfy the row. If they overlap it, it is the direction of minimum translation, and the violated row tells Phase~I in which direction to push the vehicle. An SH polygon never contains an obstacle point in its interior.

\emph{DA-ASC boxes.} A DA-ASC box starts from the seed box, the smallest $\phi_k$-aligned box that contains the inflated corners, and is discarded if the seed box meets an obstacle. The box then grows along $\pm\ehead{\phi_k}$ and $\pm\eperp{\phi_k}$ by step doubling with back-off, testing every trial box against the convex obstacles and the map boundary with the separating-axis theorem. Growth has two stages. The dynamics-aware stage first extends each side towards a target taken from the current iterate: longitudinally by $\tfrac12(|\spd_k|+|\spd_{k+1}|)\hstep+r_{\mathrm{pos}}$, the distance covered in one step, and laterally by $\max_j\sweeps{k,j}+r_{\mathrm{pos}}$, the room the corners need between samples. The second stage grows every side greedily up to $l_{\max}$ beyond the seed. Aligning the box with the vehicle has a second benefit: Lemma~\ref{lem:dir} shows that the linearisation error of the corners along the side rows of an aligned box is smaller by a factor of about four than along a row at a large angle to the heading.

\emph{Use in the two phases.} Phase~I uses SH polygons for all intervals, built around the warm start. At the start of Phase~II every interval whose seed box fits at the current, now feasible, iterate receives a DA-ASC box, and the others keep their SH polygon. On the 818 solved benchmark cases, 2.8\% of the intervals kept an SH polygon at this point. %

\emph{Rebuild rule.} After an accepted step, a corridor is rebuilt only if its smallest slack $s_k=-\max_{j,q,i}\gcon_{k,j,q,i}$ at the current iterate is below $\sact$, only if it has been rebuilt fewer than $\budget$ times, and the new corridor replaces the old one only if
\begin{equation}\label{eq:rebuild}
s_k^{\mathrm{new}}\ge s_k^{\mathrm{old}}+\sgain .
\end{equation}
In Phase~II the new corridor is a DA-ASC box, or an SH polygon where the seed box does not fit. The rule preserves feasibility and allows finitely many rebuilds (Lemmas~\ref{lem:manage} and~\ref{lem:finite}). Rebuild attempts for an interval whose corners have moved by less than 5~cm since the last rejected attempt are skipped, which saves computation and does not affect the lemmas.

\subsection{Restricted Subproblem}\label{sec:m-subproblem}
At a feasible iterate $\bar\dec$, Phase~II computes a step $\step=(\step\Xtraj,\step\Utraj,\step\tscale)$ from the quadratic program
\begin{equation}\label{eq:sub}
\begin{aligned}
\min_{\step,\tlev}\ &\nabla\cost(\bar\dec)^\top\step+\tfrac12\step^\top\Hess\step\\
\text{s.t. }&\step\state_0=0,\ \ \step\state_{k+1}=A_k\step\state_k+B_k\step\ctrl_k+C_k\step\tscale,\\
&d_i(\bar\dec)+\nabla d_i(\bar\dec)^\top\step+\restr(\tlev)\le[d_i(\bar\dec)]_+\quad\forall i,\\
&\ninf{\Dscale\step}\le\tlev\le\trr,
\end{aligned}
\end{equation}
where $A_k$, $B_k$, $C_k$ are the Jacobians of \eqref{eq:model} with respect to $\state_k$, $\ctrl_k$ and $\tscale$, the $d_i$ collect the corridor constraints \eqref{eq:corr}, the state and input bounds and the terminal box, and $\Hess$ is the Hessian of \eqref{eq:cost} in $\Utraj$ at fixed $\tscale$; the cross term between $\tscale$ and $\Utraj$ is left to the ratio test. The bounds on $\tscale$ enter without restriction because the projection does not change $\tscale$. Three features distinguish \eqref{eq:sub} from a standard SQP subproblem. There is no virtual control, because $\bar\dec$ satisfies the dynamics. The right-hand side $[d_i(\bar\dec)]_+$ absorbs constraint violations of the order of the solver tolerance, so $\step=0$ is always feasible (Lemma~\ref{lem:tol}). The restriction term $\restr(\tlev)\ge0$ vanishes at $\step=0$ and covers the second-order error of the linearisation and of the projection. With $\restr=\kQ\tlev^2$ the subproblem is a convex quadratically constrained program; the implementation uses $\restr=\kQ\trr\tlev$, which is linear in $(\step,\tlev)$ for the current radius $\trr$ and at least as conservative, so \eqref{eq:sub} is a quadratic program.

\subsection{Projection, Verification and Acceptance}\label{sec:m-accept}
The step of \eqref{eq:sub} satisfies only the linearised dynamics. The projection closes the dynamics with a time-varying LQR feedback around the linear prediction, in the spirit of the projection operator of Hauser~\cite{Hauser2002Projection}:
\begin{equation}\label{eq:proj}
\begin{aligned}
\tilde\ctrl_k&=\bar\ctrl_k+\step\ctrl_k-\Klqr{k}\big(\tilde\state_k-\bar\state_k-\step\state_k\big),\\
\tilde\state_{k+1}&=\Fmap\big(\tilde\state_k,\tilde\ctrl_k,(\bar\tscale+\step\tscale)\hzero\big),\qquad\tilde\state_0=\state_{\mathrm{init}},
\end{aligned}
\end{equation}
with gains $\Klqr{k}$ from the Riccati recursion on $(A_k,B_k)$. The trial point $\tilde\dec$ satisfies the dynamics exactly, and without feedback its deviation from the linear prediction would grow along the horizon (Lemma~\ref{lem:proj}). The a-posteriori verification then evaluates every inequality constraint at $\tilde\dec$,
\begin{equation}\label{eq:verify}
d_i(\tilde\dec)\le\tauv\quad\forall i,\qquad 0\le\tauv<\clr ,
\end{equation}
so the certificate of an accepted iterate does not rely on the constant $\kQ$ or on the QP solver's accuracy.

If \eqref{eq:verify} fails, the trial point is discarded, $\trr$ is halved and $\kQ$ is doubled; the lower bound below which $\kQ$ may later be decreased is raised to the new value (Lemma~\ref{lem:ratchet}). If it holds, the ratio $\ratio$ of actual to predicted decrease of $\cost$ decides acceptance ($\ratio\ge0.1$) and the radius update ($\trr$ is set to half the step length for $\ratio<0.25$ and doubled up to $\trr_{\max}$ for $\ratio>0.75$ when the trust region is active); after a verified step, $\kQ$ is halved if the result stays above both the lower bound and twice the observed second-order error. The objective itself is the merit function, so no penalty weight appears. In the bilinear mode a large radius inflates the restriction and can make the predicted decrease vanish at a non-critical point; a step whose predicted decrease is below $\nu\trr^2$ is therefore rejected and $\trr$ halved (Lemma~\ref{lem:descent}). Phase~II stops when the predicted decrease falls below a relative tolerance and the criticality measure $\crit$, the negative optimal value of \eqref{eq:sub} with $\kQ=0$ at the current radius, is below the same tolerance; otherwise $\trr$ is halved. It also stops when $\trr<10^{-8}$ and after 300 iterations. These are practical stopping rules, not a certificate of stationarity: both quantities scale with $\trr$ and become small at a small radius whether or not the iterate is stationary, and the implementation also stops if the subproblem for $\crit$ cannot be solved. Every returned iterate is feasible and verified; whether it is close to a KKT point is not certified.

\subsection{Phase I and Restarts}\label{sec:m-phase1}
Phase~I minimises the infeasibility $\infeas(\dec)=\|c(\dec)\|_1+\sum_i[d_i(\dec)]_+$ with $\tscale=1$, by a trust-region method whose subproblem is the linear-programming model of $\infeas$ with a small quadratic regularisation of the inputs. Its corridors are SH polygons built around the warm start and rebuilt with rule \eqref{eq:rebuild}. When the iteration stalls with $\infeas>0$, all corridors are rebuilt around the current iterate, which is usually closer to dynamic feasibility than the warm start, and Phase~I restarts, at most twice. Once $\infeas$ is below $10^{-7}$, the projection \eqref{eq:proj} and the verification \eqref{eq:verify} produce the first Phase-II iterate. If Phase~I fails, the planner retries with a slower and finer warm start from the same path (Table~\ref{tab:params}); the finer grid keeps $\hstep$, and with it the sweep margin, small as the horizon grows. %

\subsection{Structured QP and Per-Iteration Cost}\label{sec:m-qp}
Subproblem \eqref{eq:sub} couples the intervals only through $\step\tscale$ and $\tlev$. Carrying both as constant states, $\xi_k=(\step\state_k,\step\tscale,\tlev)$ with $\xi_{k+1}=\big[\begin{smallmatrix}A_k&C_k&0\\0&1&0\\0&0&1\end{smallmatrix}\big]\xi_k+\big[\begin{smallmatrix}B_k\\0\\0\end{smallmatrix}\big]\step\ctrl_k$, turns it into an optimal-control QP with stage-wise constraints; the rows of \eqref{eq:corr} at $q=k+1$ are expressed in stage-$k$ variables through the linearised dynamics. We solve it with the Riccati-based interior-point solver HPIPM~\cite{Frison2020HPIPM}, whose cost per iteration grows linearly with $\Nint$, and fall back to Clarabel~\cite{Goulart2024Clarabel} when HPIPM does not converge within 60 iterations. Obstacles enter only the corridor generation; the QP sees $\Nint$ corridors whose rows are padded to the largest row count of the horizon (four for boxes, more for SH polygons) by repeating a row, which leaves the feasible set unchanged.

\begin{algorithm}[t]
\caption{DA-ASC/FP-SQP}\label{alg:main}
\small
\begin{algorithmic}[1]
\Require start $\state_{\mathrm{init}}$, terminal box $\Xf$, obstacles, map box $\Map$
\State front end: Hybrid A* path, speed profile $\to(\Xtraj^0,\Utraj^0,\Tref)$
\State \textbf{Phase I:} SH polygons around the warm start; minimise $\infeas$ with rebuilds \eqref{eq:rebuild} and at most two restarts; on failure retry with the next warm start
\State project \eqref{eq:proj} and verify \eqref{eq:verify}; replace SH polygons by DA-ASC boxes where the seed box fits
\Repeat \Comment{Phase II: every accepted $\dec$ is certified}
\State linearise at $\dec$; solve \eqref{eq:sub} for $(\step,\tlev)$ and the predicted decrease $-m^\star$
\If{$-m^\star<\nu\trr^2$} halve $\trr$; \textbf{continue} \EndIf
\State $\tilde\dec\gets$ projection \eqref{eq:proj} of $\dec+\step$
\If{\eqref{eq:verify} fails} halve $\trr$; double $\kQ$ and raise its lower bound; \textbf{continue} \EndIf
\State $\ratio\gets(\cost(\dec)-\cost(\tilde\dec))/(-m^\star)$; update $\trr$ and $\kQ$
\If{$\ratio\ge0.1$} $\dec\gets\tilde\dec$; rebuild corridors by \eqref{eq:rebuild} \EndIf
\Until{$-m^\star$ and $\crit$ below tolerance, or iteration limit}
\State \Return $\dec$ \Comment{also returned if interrupted after any accepted step}
\end{algorithmic}
\end{algorithm}

%% file: sections/theory.tex
\section{Theoretical Properties}\label{sec:theory}

Shared corridors with the sweep margin keep the whole vehicle rectangle clear of obstacles between samples (Section~\ref{sec:t-geom}), corridor rebuilds preserve feasibility and stop after finitely many steps (Section~\ref{sec:t-rebuild}), and every Phase-II iterate is feasible while the iterates converge in the sense of Wright and Tenny~\cite{Wright2004FPSQP} (Section~\ref{sec:t-conv}). Section~\ref{sec:t-phase1} covers Phase~I and \obcact{}; proofs are in Appendix~\ref{app:proofs}.

\subsection{Continuous-Time Collision Avoidance}\label{sec:t-geom}

Throughout, the input is held constant on each interval. For interval $k$ let $\flow{k}(\tau)$, $\tau\in[0,\hstep]$, denote the exact solution of the kinematic bicycle model that starts at $\state_k$ under the input $\ctrl_k=(\acc_k,\srate_k)$, and let $\corner{\beta}(\state)=\pos+\Rot{\head}\beta$ be the world position of a body-frame point $\beta=(\beta_x,\beta_y)$. Because $\dot\spd=\acc$ and $\dot\steer=\srate$ are constant on the interval, $\spd$ and $\steer$ are affine in $\tau$, and RK4 integrates them exactly; the bounds below therefore depend only on the sampled values at $k$ and $k+1$. Define
\begin{equation}\label{eq:Mk}
\begin{aligned}
V_k &= \max(|\spd_k|,|\spd_{k+1}|),\\
K_k &= \max(|\tan\steer_k|,|\tan\steer_{k+1}|)/\wb,\\
S_k &= \max(\sec^2\steer_k,\sec^2\steer_{k+1}),\\
\Theta_{1,k} &= V_k^2K_k^2,\qquad \Theta_{2,k}=|\acc_k|K_k+V_k|\srate_k|S_k/\wb,\\
\Mb{k}(\beta) &= \big[(|\acc_k|+|\beta_x|\Theta_{1,k}+|\beta_y|\Theta_{2,k})^2\\
&\quad+(V_k^2K_k+|\beta_x|\Theta_{2,k}+|\beta_y|\Theta_{1,k})^2\big]^{1/2}.
\end{aligned}
\end{equation}

\begin{proposition}[Sweep margin]\label{prop:sweep}
Let $|\steer_k|,|\steer_{k+1}|<\pi/2$. For every body point $\beta$ and every $\tau\in[0,\hstep]$, the point $\corner{\beta}(\flow{k}(\tau))$ lies within
\begin{equation}\label{eq:sweep}
\sweep{k}(\beta)=\Mb{k}(\beta)\,\hstep^2/8
\end{equation}
of the segment joining $\corner{\beta}(\flow{k}(0))$ and $\corner{\beta}(\flow{k}(\hstep))$.
\end{proposition}

The bound follows from $\|\ddot{\corner{\beta}}\|\le\Mb{k}(\beta)$ on the interval and the interpolation error of a twice-differentiable curve. It is tight: on random intervals within the bounds of Table~\ref{tab:params}, the largest ratio of the observed deviation to $\sweep{k}(\beta_j)$ over the four corners is 0.996 (median 0.24 to 0.27, depending on $\hstep$). %
With $\hstep=\tscale\hzero$, the margin grows as $\tscale^2$, so the time-scaling variable enters the corridor constraints through $\partial\sweeps{k}/\partial\tscale=2\sweeps{k}/\tscale$. We write $\sweep{k,j}=\sweep{k}(\bodypt{j})$ for the four corners; the algorithm uses a smooth upper bound $\sweeps{k,j}\ge\sweep{k,j}$ (Section~\ref{sec:m-corridor}).

A corridor $\Corr{k}=\{\pos:\CA{k}\pos\le\Cb{k}\}$ with unit-norm rows is called \emph{valid} if $\Corr{k}\subseteq\Map$ and no obstacle meets its interior, $\Obs{m}\cap\operatorname{int}\Corr{k}=\emptyset$ for all $m$. Heading-aligned DA-ASC boxes and SH polygons are both valid by construction (Lemma~\ref{lem:manage}).

\begin{corollary}[Shared corridor]\label{cor:shared}
Let $\Corr{k}$ be valid and let $\mu'>0$. If every corner satisfies
\begin{equation}\label{eq:shared}
\CA{k}\corner{j}(\flow{k}(q))+(\sweep{k,j}+\mu')\mathbf 1\le\Cb{k},\quad j=1,\dots,4,\ q\in\{0,\hstep\},
\end{equation}
then for every $\tau\in[0,\hstep]$ the vehicle rectangle lies in $\Corr{k}$, at distance at least $\mu'$ from every obstacle and from the boundary of the map box.
\end{corollary}

The corridor contains the chord of every corner shrunk by $\sweep{k,j}$, so only the arc-to-chord deviation of Proposition~\ref{prop:sweep} remains. Sharing the corridor between the two end poses is what certifies the motion; separate corridors at $k$ and $k+1$ would certify the samples only.

\begin{remark}[RK4 transitions and the executed motion]\label{rem:rk4}
The constraints in \eqref{eq:corr} are imposed at the RK4 state $\state_{k+1}=\Fmap(\state_k,\ctrl_k,\hstep)$, with $\sweeps{k,j}\ge\sweep{k,j}$, the buffer $\clr$ and the verification tolerance $\tauv$ of \eqref{eq:verify}, whereas Corollary~\ref{cor:shared} needs them at the end point $\flow{k}(\hstep)$ of the interval motion. Let $\rkerr{k}=\max_j\|\corner{j}(\flow{k}(\hstep))-\corner{j}(\state_{k+1})\|$ be the RK4 local error lifted to the corners. If $\rkerr{k}<\clr-\tauv$, the constraints imply \eqref{eq:shared} with $\mu'=\clr-\tauv-\rkerr{k}$. We do not bound $\rkerr{k}$ a priori, and the condition restricts the step: on random intervals within the bounds of Table~\ref{tab:params}, $\rkerr{k}$ reaches 0.12~mm for $\hstep=0.4$~s and 4.1~mm for $\hstep=0.8$~s. %
On the 818 trajectories returned in the benchmark, whose steps do not exceed 0.31~s, $\rkerr{k}\le0.0032$~mm, against $\clr-\tauv\approx2$~cm.
The executed motion starts interval $k$ not at $\state_k$ but at a state that carries the propagated RK4 errors of the earlier intervals. A Gronwall bound on this difference grows exponentially with the horizon and is useless at the horizons of Section~\ref{sec:exp}, so we measure it. Let $\delta$ be the largest distance between a corner on the executed motion and the same corner on the interval motion at the same time; the executed motion keeps a clearance of at least $\clr-\tauv-\max_k\rkerr{k}-\delta$. On the 818 trajectories, $\delta\le0.03$~mm, and the executed final state lies within 0.03~mm of $\state_\Nint$. %
Both numbers come from numerical integration with a much finer step, not from validated enclosures of the flow. A certificate for the executed motion itself would need an a-priori bound on $\rkerr{k}$ and on its propagation, or validated integration in the verification step.
\end{remark}

\subsection{Corridor Rebuild}\label{sec:t-rebuild}

\begin{lemma}[Manageability]\label{lem:manage}
Let $\dec$ be an accepted Phase-II iterate that is feasible for \eqref{eq:PF} with valid corridors. A corridor produced by the rebuild rule \eqref{eq:rebuild}, or a DA-ASC box substituted for an SH polygon at the start of Phase~II, is valid, $\dec$ remains feasible after the replacement, and $\cost(\dec)$ is unchanged.
\end{lemma}

A box is accepted only if the separating-axis test finds it clear of every convex obstacle and of the map boundary, and an SH polygon puts each obstacle that meets its bounding box in the closed complement of one of its half-planes. The rebuild rule accepts a corridor only if it raises the smallest slack of \eqref{eq:corr} at the current iterate by at least $\sgain>0$, and the substituted box contains the inflated corners of both end poses by construction.

\begin{lemma}[Finitely many rebuilds]\label{lem:finite}
Each corridor is rebuilt at most $\budget$ times in Phase~II, and at most $\budget$ times between two restarts of Phase~I, which is restarted at most twice. Hence the corridors change at most $\Nint\budget$ times in Phase~II and at most $3\Nint\budget$ times in Phase~I.
\end{lemma}

The budget cannot be replaced by an argument on the slack gain alone: the slack is evaluated at the moving iterate and can fall again between two rebuilds.

\subsection{Convergence of FP-SQP}\label{sec:t-conv}

After the last rebuild the corridors are fixed, and Phase~II solves \eqref{eq:PF}: minimise $\cost(\dec)$ over $\dec=(\Xtraj,\Utraj,\tscale)$ subject to $c(\dec)=0$ (RK4 dynamics \eqref{eq:model}, initial state) and $d(\dec)\le0$ (state, input and $\tscale$ bounds, terminal box $\Xf$, corridor constraints \eqref{eq:corr}). A1 and A3--A6 hold by construction or are enforced by the implementation, and A2 holds for the bounds of Table~\ref{tab:params}.
\begin{enumerate}
\item[A1)] \emph{Trust region.} The trust region $\ninf{\Dscale\step}\le\trr$ bounds every component of the step, with $\Dscale$ diagonal and positive. (Construction.)
\item[A2)] \emph{Compactness and smoothness.} $\Fmap$, $\corner{j}$ and $\sweeps{k,j}$ are twice continuously differentiable on the closed neighbourhood $\mathcal Z_\rho$, of some radius $\rho>0$, of the set defined by the state, input and $\tscale$ bounds, with $\steer_{\max}+\rho<\pi/2$ and $\tscale_{\min}-\rho>0$. The feasible set of \eqref{eq:PF} is bounded: the corridors lie in $\Map$, and the heading, which is not wrapped, starts at a fixed value and changes by at most $\tscale_{\max}\hzero\spd_{\max}\tan\steer_{\max}/\wb$ per interval. (Bounds of Table~\ref{tab:params}; $\sweeps{k,j}$ is smooth because every square root in \eqref{eq:epss} has a positive argument.)
\item[A3)] \emph{Bounded model Hessian.} The subproblem Hessians satisfy $\|\Hess_i\|\le c_H$. ($\Hess_i=\mathrm{diag}(2\tscale_i\hzero w)$ acts on the inputs only and $\tscale$ is bounded.)
\item[A4)] \emph{Feasible start.} The first Phase-II iterate is feasible for \eqref{eq:PF}. (Projection and verification at the end of Phase~I.)
\item[A5)] \emph{Rebuilds.} Corridors are rebuilt only after accepted steps and finitely often. (Lemmas~\ref{lem:manage} and~\ref{lem:finite}.)
\item[A6)] \emph{Adaptive restriction constant.} Every trial point is verified by \eqref{eq:verify}; after a failure, $\trr$ is halved, $\kQ$ is multiplied by $\beta_\kappa>1$, and the lower bound for later decreases of $\kQ$ is raised to the new value. (Verification step.)
\end{enumerate}
For fixed $\Nint$, the projection constant of Lemma~\ref{lem:proj} is finite, because the LQR gains and the closed-loop transition matrices $\trans(k,j)$ of the projection are bounded on the compact set of A2. If, in addition, $\|\trans(k,j)\|\le c_\Phi\lambda^{k-j}$ with $\lambda\in(0,1)$ along the iterates, the constant does not grow with $\Nint$. We have checked this decay only numerically, on random trajectories (\path{tests/theory/check_d6b_projection_bound.py}); the theorem does not use it.

\begin{theorem}[Feasible iterates and convergence]\label{thm:fpsqp}
Let A1--A6 hold and consider Phase~II of Algorithm~\ref{alg:main} in exact arithmetic, with either the quadratic restriction $\restr=\kQ\tlev^2$ or the bilinear restriction $\restr=\kQ\trr\tlev$ together with the sufficient-decrease rule of Section~\ref{sec:m-accept}, and with trial points rejected like failed verifications when the projection moves them by more than half the step, $\|\tilde\dec-(\bar\dec+\step)\|_2>\tfrac12\|\step\|_2$. Then:
\begin{enumerate}
\item[(i)] every accepted iterate satisfies $c(\dec^i)=0$ and $d(\dec^i)\le0$ for its current corridors; by Corollary~\ref{cor:shared} and Remark~\ref{rem:rk4}, if in addition $\rkerr{k}<\clr$ for every $k$, each interval motion of $\dec^i$ keeps the vehicle rectangle at distance at least $\clr-\rkerr{k}$ (with $\tauv=0$ in exact arithmetic) from every obstacle and from the map boundary, and the final state lies in $\Xf$;
\item[(ii)] $\cost(\dec^{i+1})<\cost(\dec^i)$ after every accepted step, and rebuilds leave $\cost$ unchanged;
\item[(iii)] the verification fails only finitely often, and every accumulation point of the accepted iterates is either a KKT point of \eqref{eq:PF} with the final corridors or a point at which the Mangasarian--Fromovitz constraint qualification (MFCQ) fails.
\end{enumerate}
\end{theorem}

\emph{Proof idea.} Algorithm~\ref{alg:main} is an instance of the feasibility-perturbed trust-region SQP method of~\cite{Wright2004FPSQP} once three ingredients are supplied. The LQR projection closes the dynamics exactly and moves the linearised step by $O(\ninf{\Dscale\step}^2)$ (Lemma~\ref{lem:proj}, Appendix), which gives the asymptotic-exactness condition of~\cite{Wright2004FPSQP}. The restriction term covers this displacement together with the second-order remainders of $\corner{j}$ and $\sweeps{k,j}$ once $\kQ$ exceeds a finite $\kQs$, so every projected trial point is feasible (Lemma~\ref{lem:restrict}); A6 raises $\kQ$ after every failure, and a large $\kQ$ confines the step to the region where Lemma~\ref{lem:restrict} applies whatever the radius, so the failures stop (Lemma~\ref{lem:ratchet}). A Cauchy-type decrease bound for the restricted subproblem near non-KKT points replaces Lemma~3.3 of~\cite{Wright2004FPSQP} (Lemma~\ref{lem:descent}). The remaining steps of the proof of Theorem~3.8 in~\cite{Wright2004FPSQP} use only these properties, A1 and A3, and carry over unchanged; the objective itself serves as merit function, so no penalty weight is needed. %

\begin{remark}[Scope of Theorem~\ref{thm:fpsqp}]\label{rem:scope}
The theorem covers Phase~II only, in exact arithmetic, with exact subproblem solutions and without an iteration limit. It concerns the RK4 model \eqref{eq:model} and its interval motions, and it says nothing about a vehicle that deviates from the model, which Section~\ref{sec:e-closedloop} examines in closed-loop simulation. We claim only Theorem~3.8 of~\cite{Wright2004FPSQP}; its error-bound assumption, which that paper uses only to show that a perturbed step exists, is replaced here by the explicit projection of Lemmas~\ref{lem:proj} and~\ref{lem:restrict}. The KKT points in (iii) are those of \eqref{eq:PF} with the final corridors; they need not be KKT points of \eqref{eq:P}, because a corridor row can be active where no obstacle is. With solver tolerance $\taus>0$ and $0<\tauv<\clr$, the right-hand-side relaxation of \eqref{eq:sub} keeps $\step=0$ feasible, and every accepted iterate satisfies \eqref{eq:corr} without the buffer with margin $\clr-\tauv$ (Lemma~\ref{lem:tol}), so (i) and (ii) hold in this form. We do not claim (iii) for positive tolerances: the tolerances control the constraint violation, not the stationarity residual, and the stopping test of Section~\ref{sec:m-accept} does not certify stationarity. The benchmark implementation counts, but does not reject, trial points that the projection moves by more than half the step; in 43 random cases rerun with this counter (3923 logged Phase-II iterations) it never fired. %
\end{remark}

\subsection{Phase I and \obcact}\label{sec:t-phase1}

Phase~I minimises $\infeas(\dec)=\|c(\dec)\|_1+\sum_i[d_i(\dec)]_+$ by a trust-region method on its linearisation.

\begin{proposition}[Phase I]\label{prop:phase1}
After the last corridor change of Phase~I, every accumulation point of the Phase-I iterates is a first-order critical point of $\infeas$. Either $\infeas=0$ there, and Phase~II can start after the projection and verification, or the point is a local infeasible critical point.
\end{proposition}

$\infeas$ is a convex function of the smooth map $(c,d)$, so trust-region theory for composite nonsmooth problems applies~\cite[Thms.~11.2.5 and~11.5.1]{Conn2000TR}. %
Proposition~\ref{prop:phase1} does not guarantee success: a local infeasible critical point cannot be excluded from a single warm start, and in cluttered case 101 (Section~\ref{sec:e-main}) Phase~I stops with positive infeasibility after both restarts. A rebuild in Phase~I need not decrease $\infeas$, because an SH polygon is violated wherever the vehicle overlaps an obstacle; the proposition uses only the finiteness of Lemma~\ref{lem:finite}.

\obcact{} applies the construction of Corollary~\ref{cor:shared} to the exact dual formulation of \obca{}~\cite{Zhang2021OBCA}. Write the vehicle rectangle in body coordinates as $\Body=\{\xi:\Gamma\xi\le\gamma\}$ and obstacle $m$ as $\{y:A_my\le b_m\}$. For each interval $k$ and obstacle $m$, \obcact{} uses one multiplier $\lamo{m,k}\ge0$ for both end poses $q\in\{k,k+1\}$ and vehicle multipliers $\zetao_{m,k}^{q}\ge0$:
\begin{equation}\label{eq:obcact}
\begin{aligned}
&(A_m\pos_q-b_m)^\top\lamo{m,k}-\gamma^\top\zetao_{m,k}^{q}\ge\dmin+\sweep{k}(\betamax),\\
&\Gamma^\top\zetao_{m,k}^{q}+\Rot{\head_q}^\top A_m^\top\lamo{m,k}=0,\qquad \|A_m^\top\lamo{m,k}\|\le1,
\end{aligned}
\end{equation}
where $\betamax=(\max(\wb+\Lf,\Lr),\wid/2)$ collects the largest body coordinates, so that $\sweep{k}(\betamax)\ge\sweep{k,j}$ for all corners.

\begin{proposition}[\obcact]\label{prop:obcact}
If \eqref{eq:obcact} holds at $(\state_k,\ctrl_k,\flow{k}(\hstep))$ for obstacle $m$ and $\dmin>0$, the interval motion $\flow{k}$ keeps the vehicle rectangle at distance at least $\dmin$ from obstacle $m$. If it holds at the RK4 state $\state_{k+1}$ instead and $\rkerr{k}<\dmin$, the distance is at least $\dmin-\rkerr{k}$. The same holds for the map box when every corner of both end poses keeps a distance $\dmin+\sweep{k}(\betamax)$ from its sides.
\end{proposition}

The vector $w=A_m^\top\lamo{m,k}$ defines one half-plane that contains the rectangle at both end poses and whose complement contains the obstacle, and the argument of Corollary~\ref{cor:shared} applies to it. With one multiplier per sample, as in \obca{}, the two poses may lie in different half-planes; without the sweep margin, the chord is certified but the arc is not. Sections~\ref{sec:e-main} and~\ref{sec:e-eps} report the two failure modes. In the runs of Section~\ref{sec:e-eps}, with $\dmin=1$~mm, the solved trajectories have $\rkerr{k}\le0.15$~mm, so the condition $\rkerr{k}<\dmin$ holds there. %

%% file: sections/experiments.tex
\section{Experiments}\label{sec:exp}

\subsection{Setup}\label{sec:e-setup}
\emph{Cases.} The benchmark contains 820 cases: 200 random cases in each of four scenario families and the 20 cases of TPCAP~\cite{Li2023TPCAP}. The families are reverse perpendicular parking in a row of parked cars, parallel parking between parked cars along a curb, angled parking in a row of cars at 45 to 60 degrees, and cluttered areas of 30~m $\times$ 20~m with random convex obstacles and random start and goal poses. Start and goal poses are resampled until the rectangle clears every obstacle; the cases are generated from fixed seeds and exported in the TPCAP file format. %

\emph{Common inputs.} All methods use the same vehicle, bounds, RK4 model, terminal box and front end: one Hybrid A* path per case and the same first warm start derived from it (Table~\ref{tab:params}). Our planner may retry Phase~I with two slower and finer warm starts from the same path; this happened in 22 of the 820 cases, and with the first warm start alone it solves 797 cases. %
The baselines receive only the first warm start.

\emph{Baselines.} All baselines are our re-implementations in the same framework; differences from the original descriptions are listed in the supplement. \obca{}~\cite{Zhang2021OBCA,Zhang2018HOBCA} imposes the exact dual collision constraints at every sample with minimum distance $\dmin=0.02$ or $0.10$~m, the map box as an additional constraint, and the obstacles within 8~m of the warm-start path. \obcact{} (Proposition~\ref{prop:obcact}) uses $\dmin=0.02$~m, matching our buffer $\clr$, with an IPOPT CPU-time limit of 60~s; cases that fail are rerun with a limit of 3600~s.
\stc{} is a vertex-in-box variant of the safe travel corridors~\cite{Li2020STC,Li2022LIOM}: the original covers the body by two discs and boxes the disc centres against obstacles dilated by the disc radius, whereas \stc{} gives each of the four corners its own axis-aligned box per sample, built once around the warm start, with constraints at the samples only.
\corrnlp{} hands our corridors, sweep margin and buffer to IPOPT as one nonlinear program. \scvx{}~\cite{Mao2016SCvx,Mao2019SCvx,Luo2024Remarks} carries a virtual control with penalty weight $\wvirt=30$, chosen by a sweep ($\wvirt=1000$ solves 44 of 200 cases), on our corridors fixed after Phase~I; \scvxrb{} rebuilds them after every accepted iterate, in the spirit of STO~\cite{Yu2025STO}. %
The nonlinear programs are solved by IPOPT~\cite{Wachter2006IPOPT} through CasADi~\cite{Andersson2019CasADi} with MUMPS and a CPU-time limit of 300~s. %

\emph{Success.} A run succeeds if the planner returns a trajectory and an independent check confirms it is usable: each interval motion is re-integrated from the sampled state with 20 RK4 sub-steps, the vehicle rectangle at every sub-step is intersected with the original, possibly non-convex, obstacle polygons, and the corners are tested against the map box. The returned trajectory must also respect the bounds and the terminal box, and the baselines must satisfy the RK4 dynamics to $10^{-5}$; ours satisfies them by construction of the projection.
The check uses the original obstacles and shares no code with the corridors, the sweep margin or the linearisation. It samples the motion densely but does not bound it between sub-steps, so ``no penetration'' below means that none was detected at 20 sub-steps per interval. Times are optimisation times without the front end, which is common to all methods. The batch runs executed eight cases in parallel on one machine, an Apple M6 with 12 cores and 24~GB of memory under macOS~27.0, with Python~3.10.21, NumPy~2.2.6, SciPy~1.15.2, CasADi~3.8.1 (IPOPT~3.14.19 with MUMPS~5.8.2), HPIPM (commit bad1123) on BLASFEO (commit d3da74d), Clarabel~0.11.1 and Shapely~2.1.2, so times compare methods but are not absolute.

\begin{table}[t]
\caption{Vehicle and planner parameters used in all experiments}
\label{tab:params}
\centering\footnotesize
\begin{tabular}{@{}l@{\hspace{4pt}}p{0.76\columnwidth}@{}}
\toprule
Vehicle & $\wb=2.8$~m, $\Lf=0.96$~m, $\Lr=0.929$~m, $\wid=1.942$~m\\
Bounds & $|\spd|\le2.5$~m/s, $|\steer|\le0.75$~rad, $|\acc|\le1$~m/s$^2$, $|\srate|\le0.5$~rad/s\\
Objective & $w_T=1$, $(w_a,w_\omega)=(1,4)$, $\tscale\in[0.2,3]$\\
Terminal box & 0.1~m, 0.1~m, 0.05~rad, 0.05~m/s; $\steer$ free\\
Warm starts & $(v_{\mathrm{ref}},\text{time factor},\hzero)$ = $(0.8,1.5,0.4)$, $(0.6,1.8,0.3)$, $(0.5,2.2,0.25)$ [m/s, --, s], tried in this order; $\Nint\le320$\\
Corridors & $\clr=0.02$~m, $l_{\max}=4$~m, $\shR=3$~m, $r_{\mathrm{pos}}=0.5$~m, $\sact=0.1$~m, $\sgain=0.02$~m, $\budget=5$\\
Phase I & $\infeas<10^{-7}$, $\le100$ iterations, $\le2$ restarts\\
Phase II & $\trr_0=0.2$, $\trr_{\max}=2$, $\kQ^0=2$, $\beta_\kappa=2$, $\kQ\ge0.01$, $\nu=10^{-3}$, $\tauv=10^{-5}$, relative stop $10^{-5}$, $\le300$ iterations\\
Scaling $\Dscale^{-1}$ & $\state$: (1, 1, 0.2, 0.5, 0.2), $\ctrl$: (0.5, 0.25), $\tscale$: 0.2\\
Projection & $Q=\mathrm{diag}(10,10,10,1,1)$, $R=0.1I$\\
QP & HPIPM, tolerance $10^{-8}$, $\le60$ iterations; Clarabel fallback\\
\bottomrule
\end{tabular}
\end{table}

\begin{table*}[t]
\caption{Main comparison on 820 cases. Success requires the independent check to confirm a usable trajectory; ``pen.'' counts returned trajectories whose rectangle overlaps an obstacle. Times are optimisation times with eight runs in parallel. The $\Thor$ ratio is $\Thor$(ours)/$\Thor$(method), paired over the cases both solve.}
\label{tab:main}
\centering\footnotesize
\input{figs/table_main_paper.tex}

\end{table*}

\begin{figure*}[t]
\centering
\includegraphics[width=\textwidth]{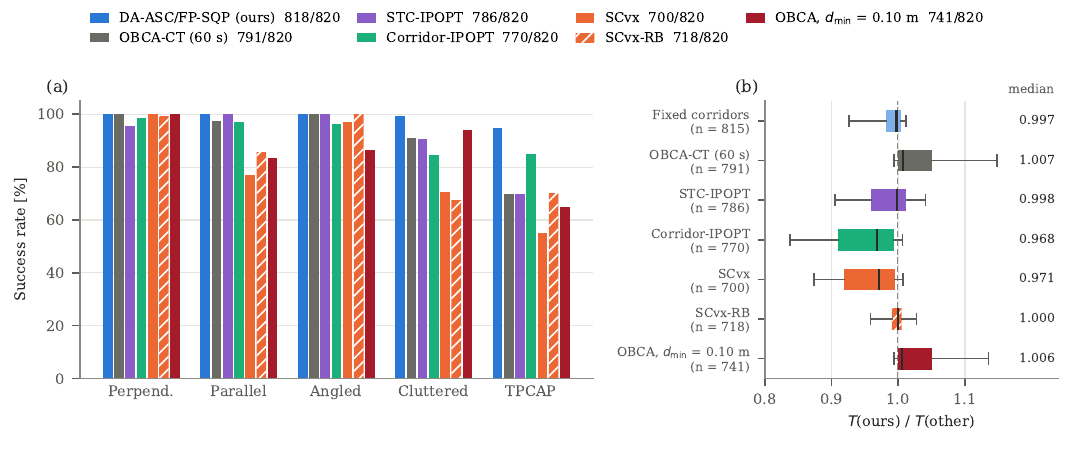}
\caption{Main comparison. (a) Success rate per scenario family (200 cases each, TPCAP 20). (b) Paired ratio $\Thor$(ours)/$\Thor$(method) on the cases both methods solve (boxes: quartiles; whiskers: 10th and 90th percentiles); values below one mean that our manoeuvre is shorter. Our planner has the highest success rate, or shares it, in every family; the manoeuvres of \obcact{} are 0.7\% shorter in the median.}
\label{fig:main}
\end{figure*}

\subsection{Main Comparison}\label{sec:e-main}
Table~\ref{tab:main} and Fig.~\ref{fig:main} summarise the benchmark. Our planner succeeds in 818 of 820 cases with no penetration; the smallest clearance of a returned trajectory is 2.27~cm, above the buffer $\clr=2$~cm. \obcact{}, the certified variant of \obca{}, succeeds in 791 cases within 60~s and in 809 with the 3600~s limit; its 29 reruns take a median of 195~s and up to 3913~s of wall time. Every case solved by \obcact{} within 60~s is also solved by our planner, and 27 cases only by ours (exact McNemar test, $p=1.5\times10^{-8}$); against the 3600-s runs the count is 10 to 1 ($p=0.012$). These counts compare complete configurations: our planner may retry Phase~I with two further warm starts, and \obcact{} receives one rerun with a longer limit. With the first warm start alone, our planner solves 797 cases, and against \obcact{} within 60~s the count is 27 to 21 ($p=0.47$); against the 3600-s runs it is 10 to 22 ($p=0.050$). The difference in success therefore comes from the retries as much as from the method, and our experiments do not isolate the contribution of FP-SQP. %

Methods that constrain only the samples return penetrating trajectories: \obca{} in 76 cases with $\dmin=0.10$~m and in 345 with $\dmin=0.02$~m, and \stc{} in 16 cases. A larger $\dmin$ reduces the penetrations but does not remove them. For \stc{}, each corner stays in its own box at the samples, but the rectangle between the corners and the motion between the samples are not constrained.

The solver matters as much as the corridors. With our corridors, sweep margin and buffer, \corrnlp{} solves 770 cases, \scvx{} 700 and \scvxrb{} 718; none of them solves a case that our planner does not solve. In their failed runs the solver stops without a point that satisfies the constraints; unlike ours, their iterates need not satisfy the constraints before convergence, and Section~\ref{sec:e-anytime} shows what this means for interrupted runs. Rebuilding the corridors at every iterate, as \scvxrb{} does, raises the success of \scvx{} in the parallel, angled and TPCAP families, lowers it in the cluttered family, and increases its median time from 2.4 to 17.2~s. %

Our manoeuvres take 0.7\% longer than those of \obcact{} in the median (791 pairs), and longer by more than 1\% in 353 pairs; the difference is largest for parallel parking, where the median ratio is 1.097 over 195 pairs. Against \corrnlp{} and \scvx{} our manoeuvres are 3.2\% and 2.9\% shorter in the median. The median optimisation time of our Python implementation is 5.8~s with a 90th percentile of 26.9~s; it does not differ significantly from that of \obcact{} (Wilcoxon signed-rank test, $p=0.61$) and is 2.4 to 3.8 times that of \stc{}, \corrnlp{}, \scvx{} and \obca{}. %

Both of our failures are outside Phase~II. In TPCAP case 7 the front end finds no path, so no method is given a warm start. In cluttered case 101, with sliver obstacles near the start and the goal, the corridors that our construction builds around the warm start leave Phase~I no feasible point, and it stops with positive infeasibility after both restarts. This is a failure of the construction and of the restoration, not evidence that no admissible corridor exists; only \obcact{} with the 3600-s limit solves this case, after 530~s. %

\subsection{Trajectories}\label{sec:e-traj}
\begin{figure*}[t]
\centering
\includegraphics[width=0.92\textwidth]{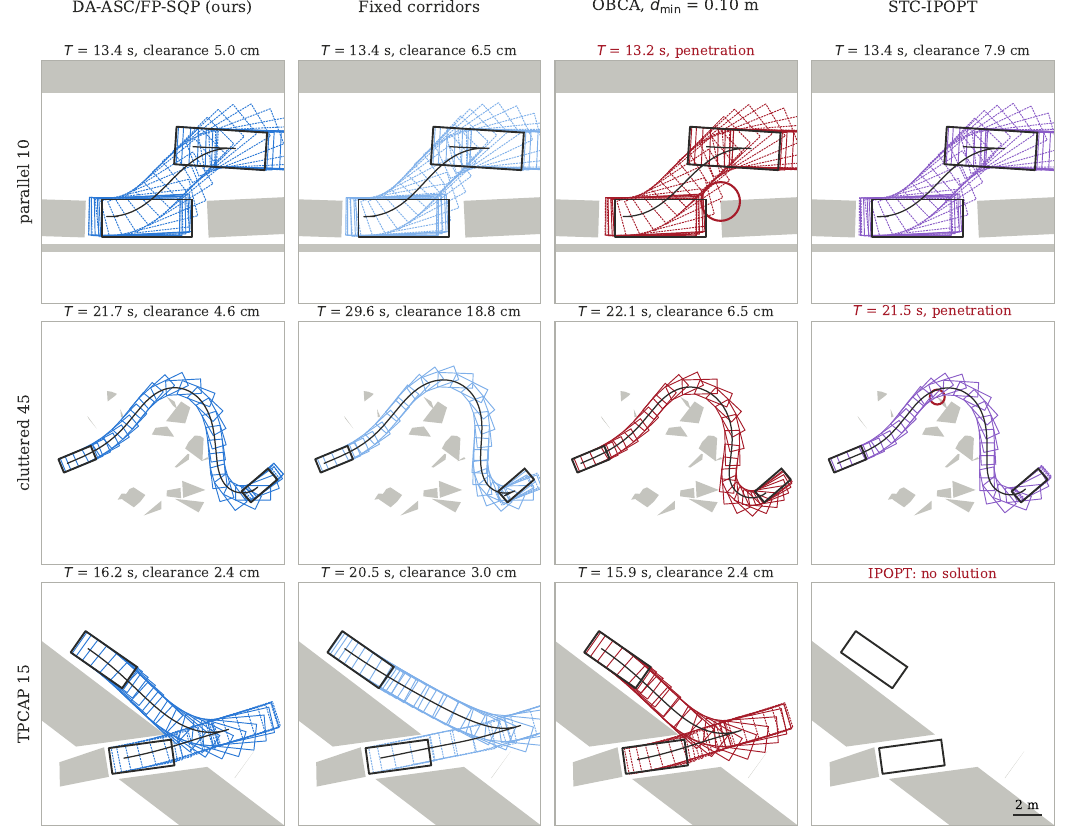}
\caption{Trajectories of four planners in three cases (vehicle outlines at every few samples; dashed: reverse gear; red: overlap with an obstacle found by the independent check). Parallel 10: \obca{} with $\dmin=0.10$~m clips the parked car between two samples. Cluttered 45: along almost the same path, the manoeuvre takes 29.6~s with fixed corridors and 21.7~s with rebuilding; \stc{} penetrates an obstacle. TPCAP 15: rebuilding shortens the manoeuvre from 20.5 to 16.2~s; \stc{} finds no solution.}
\label{fig:traj}
\end{figure*}
Fig.~\ref{fig:traj} shows three cases. In parallel case 10 all planners find nearly the same manoeuvre ($\Thor$ between 13.2 and 13.4~s), but \obca{} with $\dmin=0.10$~m lets a corner cut into the parked car between two samples. Cases cluttered 45 and TPCAP 15 show what rebuilding buys when it matters: with the corridors fixed after Phase~I, the same manoeuvre takes 36\% and 27\% longer. Such cases are a minority. %
Over all 815 cases solved with and without rebuilding, rebuilding shortens the manoeuvre by 0.3\% in the median, by more than 1\% in 274 cases and by at least 7.4\% in the best tenth, lengthens it by more than 1\% in 104 cases, and increases the median time by 18\%. It does not change the success rate significantly (3 cases to 0, $p=0.25$). %

\subsection{Certified Anytime Behaviour}\label{sec:e-anytime}
\begin{figure*}[t]
\centering
\includegraphics[width=\textwidth]{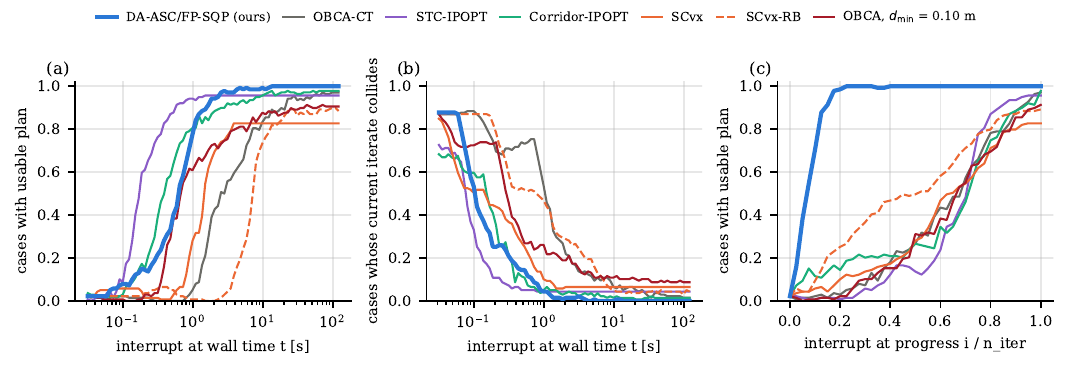}
\caption{Interrupting the planners (139 cases). (a) Share of cases whose current iterate is usable if the run is stopped at wall time $t$; (b) share whose current iterate collides; (c) as (a) against normalised progress. None of the evaluated Phase-II iterates of our planner is unusable.}
\label{fig:anytime}
\end{figure*}
To test Theorem~\ref{thm:fpsqp}(i) we recorded every iterate of every method on 139 cases (30 per random family and the 19 TPCAP cases with a front-end path). Runs with more than 100 iterates were evaluated at 100 of them, spaced both logarithmically and uniformly over the run; each evaluated iterate was executed from the start state, that is, along its executed motion, with ten RK4 sub-steps per interval (Fig.~\ref{fig:anytime}). At every sub-step the rectangle was tested against the obstacles by the separating-axis test of our corridor module, which the certificate does not use but which shares code with the corridor construction; the bounds and the terminal box were checked with a tolerance of $10^{-3}$, and containment in the map box was not checked in this experiment. For our planner this gives 10\,672 evaluated iterates out of 17\,704, and none of the 7103 evaluated Phase-II iterates was unusable; this is a sample of the Phase-II iterates, not a check of all of them.
Its unusable iterates, 1.8\% of those after the first usable one, all occur in Phase~I, where no guarantee is claimed. For the other methods, 11 to 30\% of the evaluated iterates after the first usable one are unusable again, so a run interrupted late can return a worse plan than one interrupted early. The first usable iterate appears after a median of 0.52~s for our planner, later than for \stc{} (0.14~s) and \corrnlp{} (0.25~s); what differs is not how early a usable plan appears but whether it stays usable. %

\subsection{Computational Cost}\label{sec:e-cost}
\begin{figure}[t]
\centering
\includegraphics[width=\columnwidth]{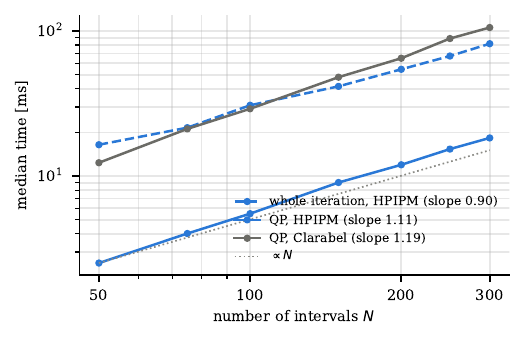}
\caption{Median time of one Phase-II QP and of one whole Phase-II iteration against the number of intervals $\Nint$ (eight cases, 15 iterations each, one thread). HPIPM grows linearly with $\Nint$; Clarabel on the same QP is five to six times slower.}
\label{fig:cost}
\end{figure}
Fig.~\ref{fig:cost} measures the cost per Phase-II iteration for $\Nint$ between 50 and 300. The time per interior-point iteration of HPIPM grows with slope 1.00 in the log--log plot, from 0.20 to 1.18~ms; one QP takes 2.5 to 18.3~ms and one whole iteration, including linearisation, projection, verification and corridor rebuild, 16.5 to 81.8~ms. Clarabel solves the same QPs 4.9 to 5.8 times more slowly. Across the benchmark, Phase~II needs a median of 71 iterations; 35 of the 818 solved cases stop at the limit of 300 iterations, with a feasible and usable trajectory that has not met the stopping test, and 24.5\% of the QPs fall back to Clarabel. The QP takes 15 to 22\% of a Phase-II iteration; the rest is linearisation, projection, verification and corridor rebuild, implemented in Python. %

\subsection{Corridor and Margin Ablation}\label{sec:e-ablation}
\begin{table}[t]
\caption{Ablation on 200 cases (seeds 0--49 of the four random families). Clearance: smallest exact clearance over all returned trajectories. $\Thor$ ratio: $\Thor$(full)/$\Thor$(variant), paired median. Iterations: median Phase-II iterations.}
\label{tab:ablation}
\centering\footnotesize\setlength{\tabcolsep}{3pt}
\resizebox{\columnwidth}{!}{\input{figs/table_ablation_paper.tex}}
\end{table}
Table~\ref{tab:ablation} removes one component at a time. Without rebuilding, and with axis-aligned instead of heading-aligned boxes, every case is still solved: where an aligned seed box does not fit, an SH polygon takes over. Both variants give slightly longer manoeuvres (median ratios 0.996 and 0.994). With the box-only Phase~I used before the SH polygons were introduced, axis-aligned boxes solved only 126 of the 200 cases; we attribute this to the seed box, which for an obliquely parked car is much larger than the car and rarely fits. Never decreasing $\kQ$ keeps every case solved but raises the median number of iterations from 71 to 169 and makes parallel parking hit the iteration limit in the median. %

The sweep margin and the buffer are the components the certificate rests on. Removing both makes 6 of the 200 trajectories penetrate an obstacle; each returned trajectory satisfies its corridor constraints at the samples, and the penetration happens between them, where the SH polygons and the grown boxes can touch an obstacle. Removing only the sweep margin leaves no penetration, but the smallest clearance drops to 1.78~cm, below the buffer, because the buffer now absorbs the sweep; removing only the buffer leaves 0.44~cm. Neither margin changes the median $\Thor$ by more than 0.2\%. %

\subsection{Sweep-Margin Counterexample}\label{sec:e-eps}
\begin{figure}[t]
\centering
\includegraphics[width=\columnwidth]{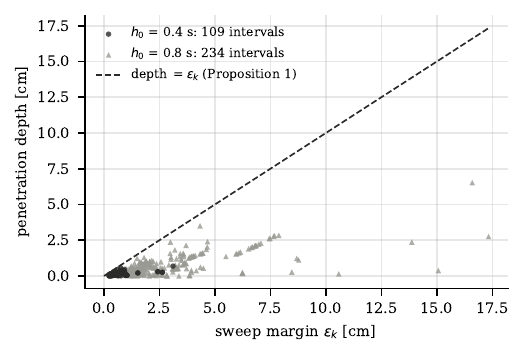}
\caption{\obcact{} with the sweep margin switched off: penetration depth of every penetrating interval against its margin $\sweep{k}$ (343 intervals, two grids). All points lie below the bound of Proposition~\ref{prop:sweep}; with the margin switched on, none of the 479 trajectories penetrates.}
\label{fig:eps}
\end{figure}
Our planner rarely drives the corridor constraints to equality at a turn, so the ablation above understates what the sweep margin does. \obcact{} with a small minimum distance ($\dmin=1$~mm) makes a sharper test: its shared hyperplane certifies the chord of every corner, and only the arc-to-chord deviation is left. We solved 250 random cases (five families including L-shaped narrow passages) with and without $\sweep{k}$ on two grids. Without the margin, 88 of 247 trajectories on the 0.4-s grid and 151 of 249 on the 0.8-s grid penetrate an obstacle, by up to 0.69 and 6.52~cm; every penetration occurs between samples and is smaller than $\sweep{k}$ (Fig.~\ref{fig:eps}), with a median depth-to-margin ratio of 0.19 and 0.31 and a maximum of 0.83. With the margin, none of the 479 trajectories penetrates. %
The median cost of the margin in $\Thor$ is 0.02\% and 0.08\% on the two grids; 12 and 19 of 239 pairs become more than 1\% longer, mostly in parallel parking. Re-solving these pairs from each other's solutions closes the gap in only 5 of 40 pairs. The remaining differences are consistent with a cost of the certificate in tight slots, but IPOPT returns local solutions, and the cross warm starts we tried cannot exclude better ones. %

\subsection{Warm-Start Robustness}\label{sec:e-init}
\begin{figure}[t]
\centering
\includegraphics[width=\columnwidth]{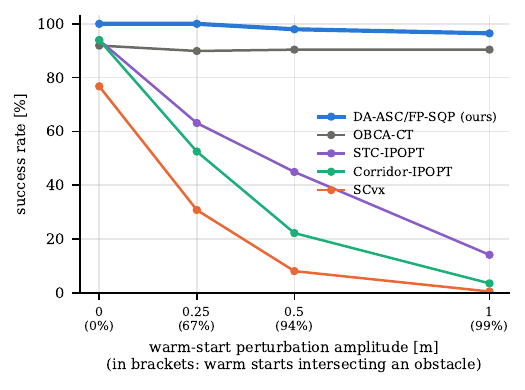}
\caption{Success rate when the front-end path is perturbed laterally (99 cases, two draws per amplitude). The share of perturbed warm starts that intersect an obstacle is given in brackets. Our planner keeps 96.5\% at 1.0~m, where 99\% of the warm starts are in collision.}
\label{fig:init}
\end{figure}
A two-stage planner depends on its warm start, and the corridor-based methods depend on it most, because their corridors are built around it. We perturbed the front-end path of 99 cases (20 per random family and the 19 TPCAP cases with a path) by smooth lateral offsets of amplitude 0.25, 0.5 and 1.0~m that vanish at the start, the goal and every cusp, with two draws per amplitude; 67\%, 94\% and 99\% of the perturbed warm starts intersect an obstacle. Fig.~\ref{fig:init} shows the success rates. Our planner solves 100\%, 98.0\% and 96.5\% of the perturbed cases, \obcact{} about 90\% at every amplitude, and the other corridor-based methods fall to between 0.5\% and 14.1\% at 1.0~m. The SH polygons are responsible: with the box-only Phase~I, which cannot place a box around a warm start in collision, our planner solved 43.9\% at 1.0~m. %

\subsection{Closed-Loop Simulation}\label{sec:e-closedloop}
\begin{table*}[t]
\caption{Closed-loop simulation on 219 cases. Error columns: plans with $\clr=2$~cm; cross-track and clearance loss (planned minus realised clearance) as 95th percentile and maximum over runs. ``pen.'': cases with a penetration in at least one run, for plans with the buffer $\clr$ [cm] and, where marked, input bounds reduced to 90\%.}
\label{tab:closedloop}
\centering\footnotesize
\input{figs/table_exp14_paper.tex}

\end{table*}
Theorem~\ref{thm:fpsqp} holds for the model. To see how much deviation from the model the buffer absorbs, we tracked the planned trajectories of 219 cases (the first 50 solved cases of each random family and the 19 solved TPCAP cases) in closed-loop simulation. The tracker applies the planned inputs with time-varying LQR feedback designed on the planning model, running at up to 20~Hz, without delay compensation, integral action or filtering. The simulated vehicle differs from the model by one factor at a time or by a combination of severity $s\in[0,1]$: wheelbase error, actuation delay, actuator limits below the planning bounds, first-order steering lag and white localisation noise. We replanned the same cases with $\clr=5$ and 10~cm and with the input bounds reduced to 90\%, all 219 of them successfully. Table~\ref{tab:closedloop} gives the results; the full set of settings and a figure are in the supplement. %

Without mismatch the tracker follows the plans to within 0.1~mm and no case penetrates. The benchmark buffer of 2~cm does not absorb realistic tracking errors: under single-factor mismatch between 1 and 32 of the 219 cases penetrate, and outside the two hardest settings every penetrating case has a planned clearance of at most 6.0~cm (median 2.9~cm). The geometry explains the sensitivity: the front corners lie 3.88~m from the rear axle, so a heading error of $0.3^\circ$ moves them by 2~cm. With $\clr=10$~cm no case penetrates for a wheelbase error between $-5\%$ and $+2.5\%$, delays up to 100~ms, actuator limits down to 90\%, noise up to 5~cm and $0.5^\circ$, steering lag up to 0.1~s and combined severity 0.25, at a median cost of 0.4\% in $\Thor$ (90th percentile 4.2\%). %
A larger buffer does not help when the actuators cannot deliver the planned inputs: with actuator limits 20\% below the planning bounds, raising $\clr$ from 2 to 10~cm reduces the penetrating cases only from 32 to 28, whereas planning with 90\% of the input bounds reduces them to 12, and to 10 together with $\clr=10$~cm. That combination leaves 3 of 219 penetrating cases under the single-factor mismatches other than the actuator deficit (one with wheelbase $+5\%$, two with steering lag 0.2~s) and 1 at combined severity 0.5, at a median cost of 1.0\% in $\Thor$ (90th percentile 7.3\%). In the two most severe settings the tracker lost the reference in some runs (4 of 219 and 94 of 1095), and we report only penetration counts there. %
These numbers characterise one tracker, the listed mismatches and 219 cases. They are the basis on which $\clr$ and the input margin are chosen for execution, not a guarantee for a physical vehicle.

%% file: figs/table_main_paper.tex
\begin{tabular}{@{}lrrrrrrrrrr@{}}
\toprule
Method & \multicolumn{2}{c}{820 cases} & \multicolumn{5}{c}{success per family [\%]} & \multicolumn{2}{c}{time [s]} & $\Thor$ ratio\\
\cmidrule(lr){2-3}\cmidrule(lr){4-8}\cmidrule(lr){9-10}
 & success & pen. & perp. & par. & ang. & clut. & TPCAP & med. & P90 & med. (pairs)\\
\midrule
\ours{} (ours) & 818 & 0 & 100 & 100 & 100 & 99.5 & 95.0 & 5.8 & 26.9 & --\\
fixed corridors & 815 & 0 & 100 & 100 & 100 & 98.0 & 95.0 & 4.9 & 21.8 & 0.997 (815)\\
\obcact{}, 60~s & 791 & 0 & 100 & 97.5 & 100 & 91.0 & 70.0 & 6.0 & 22.6 & 1.007 (791)\\
\obcact{}, 3600~s & 809 & 0 & 100 & 98.0 & 100 & 98.0 & 85.0 & 6.1 & 27.0 & 1.007 (808)\\
\stc{} & 786 & 16 & 95.5 & 100 & 100 & 90.5 & 70.0 & 1.7 & 3.1 & 0.998 (786)\\
\corrnlp{} & 770 & 0 & 98.5 & 97.0 & 96.5 & 84.5 & 85.0 & 2.0 & 4.1 & 0.968 (770)\\
\scvx{} & 700 & 0 & 100 & 77.0 & 97.0 & 70.5 & 55.0 & 2.4 & 5.3 & 0.971 (700)\\
\scvxrb{} & 718 & 0 & 99.0 & 85.5 & 100 & 67.5 & 70.0 & 17.2 & 57.9 & 1.000 (718)\\
\obca{}, $\dmin=0.10$~m & 741 & 76 & 100 & 83.5 & 86.5 & 94.0 & 65.0 & 1.5 & 6.3 & 1.006 (741)\\
\obca{}, $\dmin=0.02$~m & 472 & 345 & 97.0 & 77.5 & 0.5 & 56.5 & 45.0 & 1.5 & 4.8 & 1.012 (472)\\
\bottomrule
\end{tabular}

%% file: figs/table_ablation_paper.tex
\begin{tabular}{@{}lrrrrr@{}}
\toprule
Variant & success & pen. & min.\ clear. [cm] & $\Thor$ ratio & iterations\\
\midrule
full planner & 200/200 & 0 & 2.43 & -- & 71\\
no rebuild (fixed corridors) & 200/200 & 0 & 2.48 & 0.996 & 61\\
axis-aligned boxes & 200/200 & 0 & 2.06 & 0.994 & 77\\
no sweep margin & 200/200 & 0 & 1.78 & 1.000 & 66\\
no clearance buffer ($\clr=0$) & 200/200 & 0 & 0.44 & 1.001 & 73\\
no sweep margin, $\clr=0$ & 194/200 & 6 & $<0$ & 1.002 & 70\\
$\kQ$ never decreased & 200/200 & 0 & 2.44 & 0.996 & 169\\
\bottomrule
\end{tabular}

%% file: figs/table_exp14_paper.tex
\begin{tabular}{lrrrrrrrrrr}
\toprule
Mismatch & \multicolumn{2}{c}{cross-track [cm]} & $e_\theta$ [deg] & \multicolumn{2}{c}{clearance loss [cm]} & $\mu=2$ & $\mu=5$ & $\mu=10$ & $\mu=2$, 90\% & $\mu=10$, 90\% \\
 & P95 & max & max & P95 & max & pen. & pen. & pen. & pen. & pen. \\
\midrule
none & 0.0 & 0.0 & 0.0 & 0.0 & 0.0 & 0 & 0 & 0 & 0 & 0 \\
wheelbase $-5\%$ & 3.8 & 4.8 & 1.2 & 3.4 & 3.6 & 2 & 0 & 0 & 1 & 0 \\
wheelbase $+5\%$ & 13.4 & 19.8 & 5.0 & 6.3 & 32.3 & 8 & 4 & 1 & 7 & 1 \\
delay 50 ms & 1.4 & 6.7 & 1.4 & 1.1 & 3.1 & 1 & 0 & 0 & 0 & 0 \\
delay 100 ms & 3.7 & 16.3 & 3.6 & 2.9 & 10.0 & 5 & 1 & 0 & 2 & 0 \\
actuator limits $-5\%$ & 4.4 & 14.0 & 5.4 & 1.5 & 4.4 & 1 & 0 & 0 & 0 & 0 \\
actuator limits $-10\%$ & 16.9 & 49.2 & 16.5 & 6.8 & 22.6 & 2 & 2 & 0 & 0 & 0 \\
actuator limits $-20\%$ & \multicolumn{5}{c}{--- (tracker lost the reference in 4/219 runs)$^\dagger$} & 32 & 28 & 28 & 12 & 10 \\
noise 2 cm, $0.5^\circ$ & 2.9 & 9.5 & 2.2 & 1.9 & 4.9 & 3 & 0 & 0 & 3 & 0 \\
noise 5 cm, $0.5^\circ$ & 7.9 & 15.9 & 4.0 & 3.6 & 19.9 & 14 & 1 & 0 & 9 & 0 \\
steering lag 0.1 s & 3.8 & 17.3 & 3.4 & 2.6 & 9.4 & 3 & 1 & 0 & 1 & 0 \\
steering lag 0.2 s & 16.2 & 50.6 & 10.9 & 7.8 & 39.0 & 18 & 12 & 7 & 12 & 2 \\
combined, $s=0.25$ & 5.5 & 26.9 & 5.4 & 2.3 & 14.9 & 2 & 2 & 0 & 2 & 0 \\
combined, $s=0.5$ & 22.1 & 62.7 & 11.7 & 13.0 & 44.0 & 17 & 13 & 9 & 7 & 1 \\
combined, $s=1$ & \multicolumn{5}{c}{--- (tracker lost the reference in 94/1095 runs)$^\dagger$} & 90 & 64 & 53 & 90 & 49 \\
\bottomrule
\end{tabular}

%% file: sections/discussion.tex
\section{Discussion and Limitations}\label{sec:discussion}
\emph{Sliver obstacles.} One of our two failures, cluttered case 101, has long, thin obstacles next to the start and the goal. Phase~I stops after 24~s without a feasible point: the heading-aligned box around the vehicle meets the slivers at a small angle, and the half-planes, chosen one obstacle at a time, leave too little room to turn. \obcact{} solves this case only in the rerun with the 3600-s limit, after 530~s, so a feasible manoeuvre exists; whether an admissible corridor sequence exists we do not know. Adding half-planes to the same polygon would only shrink it; choosing them jointly so as to enlarge the polygon where the vehicle turns, or a region-growing method such as FIRI~\cite{Wang2024FIRI}, are candidates for this case. The other failure, TPCAP case 7, is a front-end failure shared by all methods. %

\emph{Solve time.} Our Python prototype needs a median of 5.8~s per case with a 90th percentile of 26.9~s, which is slower than \stc{}, \corrnlp{} and \obca{} (medians 1.5 to 2.0~s, IPOPT through CasADi) and similar to \obcact{} (6.0~s). We claim only the linear per-iteration cost of Section~\ref{sec:e-cost}; whether an implementation in a compiled language reaches the cycle times of an online planner has not been tested. %

\emph{Corridor rebuild.} Rebuilding the corridors after accepted steps changes the median horizon by 0.3\% relative to corridors fixed after Phase~I (815 pairs). The gain is concentrated: 274 cases are more than 1\% shorter and the best 10\% at least 7.4\% shorter, while 104 cases are more than 1\% longer. We have not found a property of a case that predicts the gain before solving. %

\emph{Manoeuvre duration.} The manoeuvres of \obcact{} are 0.7\% shorter in duration $\Thor$ in the median over 791 pairs, and more than 1\% shorter in 353 of them. The gap is largest for parallel parking (median ratio 1.097 over 195 pairs), where the trajectories have several gear changes and short intervals; we expect, without having tested it, that the sweep margin and the shared corridor then take a larger share of the free space. A time scale per gear segment, instead of one $\tscale$ for the whole horizon, would let short segments use shorter steps and smaller margins. %

\emph{Retries and time limits.} The success counts of Table~\ref{tab:main} compare complete configurations. Our planner retries Phase~I with two further warm starts and \obcact{} receives a rerun with 3600~s; with the first warm start alone our planner is not significantly better than \obcact{} within 60~s ($p=0.47$), and \obcact{} with 3600~s solves more cases ($p=0.050$; Section~\ref{sec:e-main}). A comparison in which every method receives the same retries and time budget would separate the method from the retry policy; we have not run it.

\emph{Stopping and optimality.} The theorem is asymptotic and assumes exact arithmetic; the implementation stops on a small predicted decrease at the current radius, on a vanishing radius or after 300 iterations (35 of 818 cases), and none of these certifies stationarity (Section~\ref{sec:m-accept}). What every returned trajectory has is feasibility and verification. A normalised first-order measure evaluated at a fixed test radius would give a stopping test that is consistent with the theorem.

\emph{Executed motion.} The certificate covers the interval motions, each started at a sampled state. The motion obtained by executing the inputs from the start state stays within 0.03~mm of them on all 818 returned trajectories, but this is a measurement (Remark~\ref{rem:rk4}); an a-priori bound on the RK4 error and its propagation, or validated integration in the verification, would turn it into a guarantee.

\emph{Phase~I.} Theorem~\ref{thm:fpsqp} covers Phase~II only. Phase~I can stop at a point that is critical for the infeasibility measure without being feasible (Proposition~\ref{prop:phase1}); the restarts and the warm-start ladder make this rarer but do not remove it, and cluttered case 101 is one instance. Proposition~\ref{prop:phase1} is also an asymptotic statement about the idealised method; the implementation caps Phase~I at 100 iterations and two restarts.

\emph{Conservatism of the sweep margin.} On typical trajectories the margin is larger than needed: in the counterexample of Section~\ref{sec:e-eps} the median ratio of penetration depth to $\sweep{k}$ is 0.19 and 0.31 on the two grids, with a maximum of 0.83. The bound itself is nearly tight, with a ratio of 0.996 on sampled intervals (Section~\ref{sec:t-geom}), so a smaller margin needs a sharper model of the corner path, for example the directional remainder of Lemma~\ref{lem:dir}, rather than a smaller constant. %

\emph{Static obstacles.} The obstacles are static and known. Moving obstacles would need corridors in space and time and a sweep margin for the relative motion, which we have not derived.

\emph{Model and execution.} The guarantees hold for the kinematic bicycle model with RK4 transitions, and we tested execution only in closed-loop simulation, not on a vehicle. Under the simulated mismatches the benchmark buffer $\clr=2$~cm is not enough (1 to 32 of 219 penetrating cases per single factor), so for execution $\clr$ must be chosen from the tracking-error budget. A larger buffer does not replace an input margin: with $\clr=10$~cm, steering lag of 0.2~s still gives 7 and actuator limits 20\% below the planning bounds 28 penetrating cases of 219. The tracker, which has no delay compensation or integral action, lost the reference in 94 of 1095 runs at the most severe combined mismatch. A tracker that accounts for delay and lag, and experiments on a vehicle, are left for future work. %

%% file: sections/conclusion.tex
\section{Conclusion}\label{sec:conclusion}
We have presented a planner for low-speed nonholonomic motion among static obstacles whose Phase-II iterates are dynamically exact and collision-free between samples for the planning model. One convex corridor per interval, holding the vehicle corners at both ends and shrunk by a sweep margin, turns inter-sample safety into linear constraints; separating-hyperplane polygons restore feasibility from warm starts in collision; and a feasibility-perturbed trust-region SQP keeps every accepted iterate usable while the corridors are rebuilt, with limit points that are KKT points of the corridor-constrained problem or violate MFCQ. On 820 cases the planner solved 818 without penetration, none of the 7103 Phase-II iterates evaluated on 139 cases was unusable, and the same sweep margin made \obca{} certifiable between samples. %
Future work covers corridors for sliver obstacles, a time scale per gear segment, recursive feasibility in a receding-horizon setting, tracking with delay and lag compensation, and experiments on a vehicle.

%% file: sections/appendix_proofs.tex
\section{Proofs}\label{app:proofs}
\setcounter{lemma}{0}\renewcommand{\thelemma}{A.\arabic{lemma}}

Vectors carry the $\infty$-norm and matrices the induced norm; the trust region uses the scaled norm $\ninf{\Dscale\,\cdot}$. $\bar\dec$ denotes a feasible Phase-II iterate with fixed corridors, $(\step,\tlev)$ a feasible point of the subproblem \eqref{eq:sub} at $\bar\dec$, and $\tilde\dec$ the trial point returned by the projection \eqref{eq:proj}.

\subsection{Proof of Proposition~\ref{prop:sweep}}
Write $e=\ehead{\head}$, $e^\perp=\eperp{\head}$ and $J$ for the rotation by $\pi/2$. With $\dot\head=\spd\tan\steer/\wb$ and $\ddot\head=\acc\tan\steer/\wb+\spd\srate\sec^2\steer/\wb$, differentiating $\corner{\beta}=\pos+\Rot{\head}\beta$ twice gives
\begin{equation*}
\ddot{\corner{\beta}}=(\acc-\ddot\head\beta_y-\dot\head^2\beta_x)\,e+(\spd\dot\head+\ddot\head\beta_x-\dot\head^2\beta_y)\,e^\perp .
\end{equation*}
Under a constant input, $\spd$ and $\steer$ are affine on the interval, so $|\spd|\le V_k$, and $|\tan\steer|$ and $\sec^2\steer$, which increase with $|\steer|$ on $(-\pi/2,\pi/2)$, are bounded by their values at the end points. Hence $\dot\head^2\le\Theta_{1,k}$, $|\ddot\head|\le\Theta_{2,k}$, $|\spd\dot\head|\le V_k^2K_k$ and $\|\ddot{\corner{\beta}}\|_2\le\Mb{k}(\beta)$ by the triangle inequality in each orthogonal component. Let $\ell(\tau)$ be the linear interpolation of $\corner{\beta}$ between $\tau=0$ and $\tau=\hstep$. For any unit vector $n$, the scalar function $n^\top(\corner{\beta}-\ell)$ vanishes at both ends and has second derivative bounded by $\Mb{k}(\beta)$, so $|n^\top(\corner{\beta}(\tau)-\ell(\tau))|\le\Mb{k}(\beta)\tau(\hstep-\tau)/2\le\sweep{k}(\beta)$. Choosing $n$ along $\corner{\beta}(\tau)-\ell(\tau)$ bounds the Euclidean distance to $\ell(\tau)$, which lies on the segment. \hfill$\square$

\subsection{Proof of Corollary~\ref{cor:shared}}
Fix $\tau$ and a corner $j$, and let $\ell_j(\tau)$ be the point of the chord given by Proposition~\ref{prop:sweep}. The end points of the chord satisfy $\CA{k}y\le\Cb{k}-(\sweep{k,j}+\mu')\mathbf 1$, a convex set, so $\ell_j(\tau)$ does too; since the rows of $\CA{k}$ have unit norm and $\|\corner{j}(\flow{k}(\tau))-\ell_j(\tau)\|_2\le\sweep{k,j}$, the corner satisfies $\CA{k}\corner{j}\le\Cb{k}-\mu'\mathbf 1$. The vehicle rectangle at time $\tau$ is the convex hull of its four corners, so every point $y$ of it satisfies the same inequality, and the ball of radius $\mu'$ around $y$ lies in $\Corr{k}\subseteq\Map$. An obstacle point $o$ does not lie in $\operatorname{int}\Corr{k}$, so some row satisfies $a_i^\top o\ge b_i$ and $\|y-o\|_2\ge a_i^\top(o-y)\ge\mu'$. \hfill$\square$

\subsection{Restriction, projection and adaptive constant}

\begin{lemma}[Projection]\label{lem:proj}
Let $\bar\dec$ be feasible, let $\Klqr{k}$ be the gains of the projection, $\bar K=\max_k\|\Klqr{k}\|$, and let $\Ltwo$ bound the second derivative of $\Fmap$ on the set $\mathcal Z_\rho$ of A2. Define $e_k=\tilde\state_k-(\bar\state_k+\step\state_k)$. Then $e_0=0$ and
\begin{equation}\label{eq:err}
e_{k+1}=(A_k-B_k\Klqr{k})e_k+R_k ,
\end{equation}
where $R_k$ is the second-order Taylor remainder of $\Fmap$ at $(\bar\state_k,\bar\ctrl_k,\bar\tscale)$ evaluated at $(\tilde\state_k-\bar\state_k,\tilde\ctrl_k-\bar\ctrl_k,\step\tscale)$. With $\tsum=\max_k\sum_{j<k}\|\trans(k,j+1)\|$ and $\kproj=\tfrac12\Ltwo(2+\bar K)^2\tsum$, every step with $\ninf{\step}\le1/\kproj$ satisfies $\max_k\|e_k\|\le\kproj\ninf{\step}^2$ and $\ninf{\tilde\dec-(\bar\dec+\step)}\le(1+\bar K)\kproj\ninf{\step}^2$. $\tsum$ is finite for fixed $\Nint$; if $\|\trans(k,j)\|\le c_\Phi\lambda^{k-j}$ with $\lambda<1$, then $\tsum\le c_\Phi/(1-\lambda)$ independently of $\Nint$.
\end{lemma}
\begin{proof}
Feasibility of $\bar\dec$ gives $\bar\state_{k+1}=\Fmap(\bar\state_k,\bar\ctrl_k,\bar\tscale\hzero)$, and the linearised dynamics give $\bar\state_{k+1}+\step\state_{k+1}=\bar\state_{k+1}+A_k\step\state_k+B_k\step\ctrl_k+C_k\step\tscale$. The projection sets $\tilde\ctrl_k=\bar\ctrl_k+\step\ctrl_k-\Klqr{k}e_k$ and rolls out $\Fmap$ with $\bar\tscale+\step\tscale$; substituting $\tilde\state_k-\bar\state_k=\step\state_k+e_k$ and $\tilde\ctrl_k-\bar\ctrl_k=\step\ctrl_k-\Klqr{k}e_k$ into the Taylor expansion yields \eqref{eq:err}. Solving the recursion, $e_k=\sum_{j<k}\trans(k,j+1)R_j$. While $\|e_j\|\le\ninf{\step}$, the argument of $R_j$ is bounded by $(2+\bar K)\ninf{\step}$, so $\|R_j\|\le\tfrac12\Ltwo(2+\bar K)^2\ninf{\step}^2$ and $\|e_k\|\le\kproj\ninf{\step}^2\le\ninf{\step}$ for $\ninf{\step}\le1/\kproj$, which closes the induction. The input displacement is $\Klqr{k}e_k$. If $\|\trans(k,j)\|\le c_\Phi\lambda^{k-j}$, the sum is a geometric series.
\end{proof}
Without feedback ($\Klqr{k}=0$), $\trans$ is the open-loop transition matrix; position integrates heading, so $\tsum$ grows at least linearly with $\Nint$. This is why the projection uses the TV-LQR gains.

\begin{lemma}[Restriction]\label{lem:restrict}
There are $\kQs<\infty$ and $r_{\mathrm p}>0$, depending only on the bounds of A2, on $\Nint$ and on $\Dscale$, such that for $\kQ\ge\kQs$ and $\trr\le r_{\mathrm p}$ the trial point satisfies $c(\tilde\dec)=0$ and $d(\tilde\dec)\le0$. The step $\step=0$, $\tlev=0$ is feasible for \eqref{eq:sub} for every $\trr>0$.
\end{lemma}
\begin{proof}
The rollout satisfies the dynamics and the initial state by construction. For an inequality $d_i$ (corridor rows \eqref{eq:corr}, state and input bounds, terminal box), split
$d_i(\tilde\dec)=d_i(\bar\dec)+\nabla d_i^\top\step+[d_i(\bar\dec+\step)-d_i(\bar\dec)-\nabla d_i^\top\step]+[d_i(\tilde\dec)-d_i(\bar\dec+\step)]$.
By A2 the first bracket is at most $\tfrac12L_d\ninf{\step}^2$ and the second at most $G_d(1+\bar K)\kproj\ninf{\step}^2$ (Lemma~\ref{lem:proj}), with $L_d$ and $G_d$ defined below. Since $\ninf{\step}\le c_D\tlev$ with $c_D=\|\Dscale^{-1}\|$, the sum is at most $\kQs\tlev^2$ with $\kQs=c_D^2(\tfrac12L_d+G_d(1+\bar K)\kproj)$. Both restriction terms satisfy $\restr\ge\kQ\tlev^2$ because $\tlev\le\trr$, so the constraint of \eqref{eq:sub} gives $d_i(\tilde\dec)\le[d_i(\bar\dec)]_+=0$. Take $r_{\mathrm p}=\min\{1/(2c_D(1+\bar K)\kproj),\,2\rho/(3c_D)\}$. Then $\ninf{\step}\le c_Dr_{\mathrm p}\le1/\kproj$, Lemma~\ref{lem:proj} applies, $\ninf{\tilde\dec-(\bar\dec+\step)}\le\tfrac12\ninf{\step}$, and $\bar\dec+\step$, $\tilde\dec$ and the segments used in the Taylor expansions lie within $\tfrac32c_Dr_{\mathrm p}\le\rho$ of the feasible set, where $L_d$ and $G_d$, the bounds on the second and first derivatives of the $d_i$ over $\mathcal Z_\rho$, apply. For $\step=0$ every restriction term vanishes and the constraints reduce to $d(\bar\dec)\le[d(\bar\dec)]_+$.
\end{proof}
For the corridor rows, $L_d$ contains the curvature of $\corner{j}$ in $\head$; Lemma~\ref{lem:dir} shows that heading-aligned corridors reduce the part of it that acts on the side rows.

\begin{lemma}[Finitely many verification failures]\label{lem:ratchet}
Under A1, A2 and A6, the verification \eqref{eq:verify}, together with the displacement test of Theorem~\ref{thm:fpsqp}, fails at most $\lceil\log_{\beta_\kappa}(\kappa^\dagger/\kappa_{\mathrm{floor}}^0)\rceil+1$ times, where $\kappa^\dagger<\infty$ is given in the proof and $\kappa_{\mathrm{floor}}^0$ is the initial lower bound on $\kQ$. No relation between $\trr_{\max}$ and $r_{\mathrm p}$ is needed.
\end{lemma}
\begin{proof}
Let $c_s=\max_i\sup\{-d_i(\dec):\dec\text{ feasible for \eqref{eq:PF}}\}$, finite by A2, and let $G$ bound $\|\nabla d_i\|_1$ on the feasible set. Take any corridor row $i$; both restriction terms satisfy $\restr\ge\kQ\tlev^2$ because $\tlev\le\trr$, so its constraint in \eqref{eq:sub} gives $\kQ\tlev^2\le-d_i(\bar\dec)-\nabla d_i^\top\step\le c_s+Gc_D\tlev$, hence $\tlev\le\bar t(\kQ)=\big(Gc_D+(G^2c_D^2+4\kQ c_s)^{1/2}\big)/(2\kQ)$, which decreases to zero as $\kQ\to\infty$. Choose $\hat\kappa$ with $\bar t(\hat\kappa)\le r_{\mathrm p}$ and set $\kappa^\dagger=\max(\kQs,\hat\kappa)$. For $\kQ\ge\kappa^\dagger$, every feasible point of \eqref{eq:sub} has $\tlev\le r_{\mathrm p}$, whatever $\trr$, and the proof of Lemma~\ref{lem:restrict} shows that the trial point passes the verification; it also passes the displacement test if $r_{\mathrm p}$ is taken smaller by the dimension factor between $\ninf{\cdot}$ and $\|\cdot\|_2$. A failure therefore implies $\kQ<\kappa^\dagger$. After a failure, $\kQ$ and its lower bound both equal $\beta_\kappa$ times the value of $\kQ$ before, and later decreases stop at the lower bound. After $n$ failures the lower bound is at least $\beta_\kappa^{n}\kappa_{\mathrm{floor}}^0$, which reaches $\kappa^\dagger$ for the stated $n$. Since $\kQ$ increases only after failures, it is also bounded above.
\end{proof}
Neither $\kQs$ nor $r_{\mathrm p}$ is computed; the lemmas need only that they exist. On the 818 solved benchmark cases, the verification failed 257 times in 77\,982 Phase-II iterations, at most five times in one case. %

\subsection{Decrease of the restricted subproblem}

\begin{lemma}[Decrease]\label{lem:descent}
Let $\dec^\star$ be feasible for \eqref{eq:PF}, satisfy MFCQ and not be a KKT point, and let $\kQ\le\bar\kappa$. There are a neighbourhood $\mathcal N$ of $\dec^\star$ and constants $\bar\varepsilon,\hat t,\hat r>0$ such that for every feasible $\bar\dec\in\mathcal N$ and every $\trr\le\hat r$ the optimal value $m^\star(\trr)$ of the bilinear-mode subproblem satisfies
\begin{equation}\label{eq:cauchy}
m^\star(\trr)\le-\tfrac{\bar\varepsilon}{4}\min(\trr,\hat t).
\end{equation}
For the quadratic mode, \eqref{eq:cauchy} holds for every $\trr>0$ with $\hat t$ replaced by $\min(\hat t,\gamma/(2\bar\kappa))$.
\end{lemma}
\begin{proof}
Because $\dec^\star$ is not KKT, some $w$ with $\ninf{\Dscale w}\le1$ satisfies the linearised dynamics, $\nabla d_i^\top w\le0$ for the active constraints and $\nabla\cost^\top w<0$; MFCQ supplies $v$ with the linearised dynamics and $\nabla d_i^\top v<0$ for the active constraints. A convex combination, normalised to $\ninf{\Dscale d}=1$, gives a direction $d$ at $\dec^\star$ with $\nabla\cost^\top d<0$ and $\nabla d_i^\top d<0$ for the active constraints. The dynamics Jacobian has full row rank (it contains the identity blocks of $\step\state_{k+1}$ and $\step\state_0$), so the projection of $d$ onto the null space of the linearised dynamics at $\bar\dec$ defines a direction $d(\bar\dec)$ that depends continuously on $\bar\dec$ and equals $d$ at $\dec^\star$. After rescaling, and shrinking $\mathcal N$ if necessary, $d(\bar\dec)$ satisfies the linearised dynamics at $\bar\dec$, $\ninf{\Dscale d(\bar\dec)}\le1$, $\nabla\cost^\top d(\bar\dec)\le-\bar\varepsilon/2$ and $\nabla d_i^\top d(\bar\dec)\le-\gamma/2$ for the constraints active at $\dec^\star$, for every $\bar\dec\in\mathcal N$; by continuity the constraints inactive at $\dec^\star$ satisfy $d_i\le-\delta/2$ and $\|\nabla d_i\|\le G$ on $\mathcal N$~\cite[Lemmas~2.1, 3.1]{Wright2004FPSQP}. %
Set $\step=s\,d(\bar\dec)$, $\tlev=s\le\trr$. For an active constraint, $d_i+s\nabla d_i^\top d+\kQ\trr s\le s(-\gamma/2+\bar\kappa\trr)\le0$ when $\trr\le\hat r=\gamma/(2\bar\kappa)$; for an inactive one, $-\delta/2+s(G+\bar\kappa\hat r)\le0$ when $s\le\delta/(2(G+\bar\kappa\hat r))$. The model value is at most $-s\bar\varepsilon/2+c_Hs^2/2\le-s\bar\varepsilon/4$ for $s\le\bar\varepsilon/(2c_H)$. Taking $s=\min(\trr,\hat t)$ with $\hat t=\min(\delta/(2(G+\bar\kappa\hat r)),\bar\varepsilon/(2c_H))$ proves \eqref{eq:cauchy}; for the quadratic mode, $\kQ s^2\le\gamma s/2$ replaces the condition on $\trr$.
\end{proof}
In the bilinear mode the restriction grows with $\trr$, and for $\trr>\hat r$ the subproblem may return $m^\star=0$ although $\bar\dec$ is not critical. The sufficient-decrease rule of Section~\ref{sec:m-accept} rejects a step with $-m^\star<\nu\trr^2$ and halves $\trr$; for $\trr\le\tilde r=\min(\hat r,\bar\varepsilon/(4\nu),(\bar\varepsilon\hat t/(4\nu))^{1/2})$, \eqref{eq:cauchy} gives $-m^\star\ge\nu\trr^2$, so the rule never rejects a step there.

\begin{lemma}[Tolerances]\label{lem:tol}
Let the subproblem be solved to constraint tolerance $\taus$, let its right-hand side be $[d(\bar\dec)]_+$, and accept a trial point only if the corridor constraints satisfy $\gcon\le\tauv$ with $0\le\tauv<\clr$. Then $\step=0$ is feasible for \eqref{eq:sub}, and every accepted iterate satisfies $\ghat=\gcon-\clr\le\tauv-\clr<0$.
\end{lemma}
\begin{proof}
The first claim is the last statement of Lemma~\ref{lem:restrict}, which does not use $d(\bar\dec)\le0$. The second is the acceptance test.
\end{proof}
Safety thus rests on the verification and not on the accuracy of the QP solver; the accuracy affects only optimality.

\subsection{Directional remainder}

\begin{lemma}[Heading-aligned rows]\label{lem:dir}
For a body point $\beta$, a heading change $\Delta\head$ and a unit row $a$, the remainder of the first-order expansion of $\corner{\beta}$ in $\head$ satisfies
\begin{equation*}
|a^\top\eta|\le\tfrac12\Delta\head^2\,|a^\top\Rot{\head}\beta|+\tfrac16|\Delta\head|^3\,|a^\top\Rot{\head}J\beta| .
\end{equation*}
\end{lemma}
\begin{proof}
$\Rot{\head+\Delta\head}=\Rot{\head}(\cos\Delta\head\,I+\sin\Delta\head\,J)$ gives $\eta=\Rot{\head}[(\cos\Delta\head-1)\beta+(\sin\Delta\head-\Delta\head)J\beta]$; bound the two scalar factors by $\Delta\head^2/2$ and $|\Delta\head|^3/6$.
\end{proof}
For a side row of a box aligned with the heading, $a=\pm\eperp{\head}$ and $|a^\top\Rot{\head}\beta|=|\beta_y|=\wid/2$, so the second-order coefficient is $\wid/4\approx0.49$~m for every corner, against $\|\beta\|_2/2\approx1.94$~m for the front corners when the row is not aligned. %

\subsection{Proof of Theorem~\ref{thm:fpsqp}}
(i) By A4 the first iterate is feasible. A trial point that fails the verification is discarded, and one that passes satisfies $c=0$ and $d\le0$; rebuilds keep feasibility (Lemma~\ref{lem:manage}). Corollary~\ref{cor:shared} with Remark~\ref{rem:rk4} ($\tauv=0$) gives the clearance, and the terminal box is one of the constraints. (ii) A step is accepted only if the predicted decrease $-m^\star$ is positive and the ratio of actual to predicted decrease exceeds $\eta_1>0$; rebuilds leave $\cost$ unchanged. (iii) By Lemmas~\ref{lem:finite} and~\ref{lem:ratchet}, after finitely many iterations the corridors and the lower bound of $\kQ$ no longer change, $\kQ\le\bar\kappa$ for some $\bar\kappa$, and no trial point fails the verification or the displacement test. From then on, Algorithm~\ref{alg:main} is the method of~\cite{Wright2004FPSQP} applied to \eqref{eq:PF}: feasibility of the perturbed point follows from Lemma~\ref{lem:restrict}, asymptotic exactness, $\ninf{\tilde\dec-(\bar\dec+\step)}\le\phi(\ninf{\step})\ninf{\step}$ with $\phi(s)=\min(\tfrac12,(1+\bar K)\kproj s)$, in the Euclidean norm of~\cite{Wright2004FPSQP}: for short steps from Lemma~\ref{lem:proj}, with the constant multiplied by a factor that depends only on the dimension, and for all other accepted steps from the displacement test, and the decrease bound of their Lemma~3.3 is replaced by Lemma~\ref{lem:descent} together with the sufficient-decrease rule. Their proof of Theorem~3.8 uses the decrease bound only in a neighbourhood of an accumulation point, in the form \eqref{eq:cauchy}, and otherwise only A1--A3 and the two properties above. It therefore applies to the tail of the sequence. \hfill$\square$

\subsection{Proof of Proposition~\ref{prop:obcact}}
Let $w=A_m^\top\lamo{m,k}$, so $\|w\|_2\le1$. For a body point $\xi\in\Body$ at pose $q\in\{k,k+1\}$, $y=\pos_q+\Rot{\head_q}\xi$, the equality constraint and $\zetao^q_{m,k}\ge0$ give
$w^\top y-\lamo{m,k}^\top b_m=(A_m\pos_q-b_m)^\top\lamo{m,k}-(\zetao^q_{m,k})^\top\Gamma\xi\ge(A_m\pos_q-b_m)^\top\lamo{m,k}-\gamma^\top\zetao^q_{m,k}\ge\dmin+\sweep{k}(\betamax)$.
Both end poses therefore lie in the closed half-plane $H=\{y:w^\top y\ge\lamo{m,k}^\top b_m+\dmin+\sweep{k}(\betamax)\}$, and so do the chords of the corners. Since $\Mb{k}(\beta)$ increases with $|\beta_x|$ and $|\beta_y|$, Proposition~\ref{prop:sweep} places each corner within $\sweep{k}(\betamax)$ of its chord, so the rectangle stays in $\{y:w^\top y\ge\lamo{m,k}^\top b_m+\dmin\}$. Every obstacle point satisfies $w^\top o=\lamo{m,k}^\top A_mo\le\lamo{m,k}^\top b_m$, which proves disjointness with separation at least $\dmin$. The map box is convex, and the same argument applies to its four sides. \hfill$\square$